\documentclass[12pt,a4paper]{ouparticle}

\usepackage{stmaryrd}
\usepackage{listings}
\usepackage{xcolor}
\usepackage[utf8]{inputenc}
\usepackage{amsthm}
\usepackage[framemethod=default]{mdframed}
\usepackage[normalem]{ulem}

\lstdefinelanguage{Isabelle}
{
	basicstyle=\ttfamily\small,
	keywords=[1]{
		definition, lemma, using, by, unfolding, named_theorems, declare, consts, axiomatization, typedecl, type_synonym, oops, theorem, abbreviation, begin, theory, imports, end, sledgehammer
	},
	keywordstyle=[1]\bfseries,
	keywords=[2]{where, infixr, assumes, shows, and},
	keywordstyle=[2]\bfseries,
	keywords=[3]{v, p, r},
	keywordstyle=[3],
	keywords=[4]{add},
	keywordstyle=[4],
	sensitive=false,
	morestring=[b]',
	morecomment=[l]{--}
}
\begin{document}

\theoremstyle{definition}
\newtheorem{definition}{Definition}[section]
\newtheorem{theorem}{Theorem}[section]
\newtheorem{corollary}{Corollary}[theorem]
\newtheorem{lemma}[theorem]{Lemma}

\newcommand{\logikey}{\textsc{LogiKEy}}

\title{\LARGE Automating Public Announcement Logic \\ with Rela\-tivized Common Knowledge \\ as a Fragment of HOL in LogiKEy}

\author{%
\name{Christoph Benzmüller}
\address{University of Bamberg, AI Systems Engineering, Bamberg, Germany \&
FU Berlin, Dep.~of Mathematics and Computer Science, Berlin, Germany}
\email{c.benzmueller@fu-berlin.de}
\and
\name{Sebastian Reiche\thanks{Corresponding Author}}
\address{FU Berlin, Dep.~of Mathematics and Computer Science, Berlin, Germany}
\email{sebastian.reiche@fu-berlin.de}
}

\abstract{A shallow semantical embedding for public announcement logic with relativized common knowledge is presented. This embedding enables the first-time automation of this logic with off-the-shelf theorem provers for classical higher-order logic. It is demonstrated (i) how meta-theoretical studies can be automated this way, and (ii) how non-trivial reasoning in the target logic (public announcement logic), required e.g.~to obtain a convincing encoding and automation of the wise men puzzle, can be realized.

Key to the presented semantical embedding is that evaluation domains are modeled explicitly and treated as an additional parameter in the encodings of the constituents of the embedded target logic;
in previous related works, e.g.~on the embedding of normal modal logics, evaluation domains were implicitly shared between meta-logic and target logic.

The work presented in this article constitutes an important addition to the pluralist \logikey\ knowledge engineering methodology, which enables experimentation with logics  and their  combinations,  with  general and  domain  knowledge,  and  with  concrete use cases — all at the same time.
}

\date{\today}


\maketitle

\section{Introduction}
\label{sec1}

Previous work has studied the application of a  universal (meta-)logical reasoning approach \cite{J41,J44} for solving  a prominent riddle in epistemic reasoning, known as the \textit{wise men puzzle}, on the computer \cite{J44}.
The solution presented there puts a particular emphasis on the adequate modeling of (ordinary) common knowledge and it also illustrates the elegance and the practical relevance of shallow\footnote{Shallow semantical embeddings are different from \emph{deep embeddings} of an object logic. In the latter case the syntax of the object logic is represented using an inductive data structure (e.g., following the definition of the language).
The semantics of a formula is then evaluated by recursively traversing the data structure, and additionally a proof theory for the logic may be encoded. Deep embeddings typically require technical inductive proofs, which hinder proof automation, that can be avoided when shallow
semantical embeddings are used instead.
For more information on shallow and deep embeddings we refer to the literature \cite{DeepShallow,DeepShallow2}.} semantical embeddings (SSEs; cf.~\cite{J23,J41}) of non-classical `object' logics in classical higher-order logic (HOL, aka Church's simple type theory; cf.~\cite{J43}), when being utilized within modern proof assistant systems such as Isabelle/HOL \cite{isabelle}.
However, this work nevertheless falls short, since it did not convincingly  address the interaction dynamics between the involved agents.
To this end, we extend and adapt in this article the universal (meta-)logical reasoning approach for \textit{public announcement logic}, and we demonstrate how it can be utilized to achieve a convincing encoding and automation of the wise men puzzle in Isabelle/HOL, so that also the interaction dynamics as given in the scenario is adequately addressed.
In more general terms, we present the first automation of public announcement logic (PAL) with relativized common knowledge, and we demonstrate that, and how, this logic can be seen and handled as a fragment of HOL. 
Key to the presented extension of the shallow semantical embedding approach is that the evaluation domains of the embedded target logic (PAL with relativized common knowledge) are no longer implicitly shared with the meta-logic HOL, and are instead now explicitly modeled as an additional parameter in the encoding of the embedded logics constituents.
We expect this approach of dynamizing shallow semantic embeddings of object logics in HOL to be applicable and adaptable to a wide spectrum of related works; cf.~\cite{sep-dynamic-epistemic,van2007dynamic} and the references therein.

The work presented in this article constitutes an important addition to the pluralist \logikey\ approach and methodology \cite{J48,J53}. \logikey's unifying formal framework is fundamentally based on SSEs of `object' logics (and their combinations) in HOL, enabling the
provision of powerful tool support \cite{B5}: off-the-shelf
theorem provers and model finders for HOL (as provided in Isabelle/HOL) are assisting the \logikey\
knowledge engineer to \textit{flexibly experiment} with underlying
logics and their combinations, with general and
domain knowledge, and with concrete use cases---all at the same
time. Continuous improvements of these off-the-shelf provers, without further ado, boost the reasoning performance in \logikey. 
Of course, specific PAL theorem provers (e.g.~\cite{DBLP:conf/tableaux/BalbianiDHL07}) are still expected to outperform the generic reasoning tools in \logikey. However, both approaches can be merged in the future and specialist provers can/should be added, e.g., as oracles to the LogiKEy framework. Regarding flexibility, there is clearly an advantage on the side of the \logikey\ approach, and it should actually be straightforward to adapt the SSE we present in this article (in future work) also for \textit{arbitrary announcement} operators as provided in APAL \cite{DBLP:conf/tableaux/BalbianiDHL07} or to DEL with action models \cite{sep-dynamic-epistemic}.

This article is structured as follows:
\textsection 2 briefly recaps HOL (Church's simple type theory), and \textsection 3 sketches PAL with relativized common knowledge. 
The main contribution of this article is presented in \textsection 4: a shallow semantical embedding of PAL with relativized common knowledge in HOL. 
The soundness of this embedding is subsequently proved in \textsection 5. Automation aspects are studied
in 
\textsection 6. This includes meta-level reasoning and completeness studies (in \textsection 6.1), the exploration of failures of uniform substitution (in \textsection 6.2), and finally an application of our framework to obtain an adequate modeling and automation of the prominent
 wise men puzzle (in \textsection 6.3).
\textsection 7 discusses related work and \textsection 8 concludes the article.

This article extends and improves our prior paper \cite{C90} in various respects. In addition to an overall improved presentation, we add a soundness proof, we show how completeness can be ensured, we modularize our embedding and its encoding in Isabelle/HOL, and we further improve it by parameterizing the notions of shared and distributed knowledge over arbitrary groups of agents; moreover, we now prove the wise men puzzle automatically also for four agents.

\section{Classical Higher-Order Logic}
\label{sec2}
We briefly recap classical higher-order logic (HOL), respectively Church's  \textit{simple theory of types} \cite{church1940,J43}, which is a logic defined on top of the simply typed lambda calculus.
The presentation is partly adapted from  Benzm\"uller~\cite{J31}. For further information on the syntax and semantics of HOL we refer to \cite{J6}.

\paragraph{Syntax of HOL.}
We start out with defining 
the set $\mathcal{T}$ of \textit{simple types} by the following abstract grammar:
$\alpha , \beta := o \ | \ i \ | \ (\alpha \rightarrow \beta)$.
Type $o$ denotes a bivalent set of truth values, containing \textit{truth} and \textit{falsehood}, and $i$ denotes a non-empty set of individuals.\footnote{In this article, we will actually associate type $i$ later on with the domain of possible worlds.} Further base types are optional.
$\rightarrow$ is the function type constructor, such that $(\alpha \rightarrow \beta) \in \mathcal{T}$ whenever $\alpha, \beta \in \mathcal{T}$. We may generally omit parentheses.
	
The \textit{terms} of HOL are defined  by the following abstract grammar:
$$s,t := p_\alpha \ | \ X_\alpha \ | \ (\lambda X_\alpha s_\beta)_{\alpha \rightarrow \beta} \ | \ (s_{\alpha\rightarrow\beta}t_\alpha)_\beta$$
where $\alpha, \beta, o \in \mathcal{T}$. The $p_\alpha \in C_\alpha$ are typed constants  and the $X_\alpha \in V_\alpha$ are typed variables (distinct from the $p_\alpha$).
If $s_{\alpha \rightarrow \beta}$ and $t_\alpha$ are HOL terms of types $\alpha \rightarrow \beta$ and $\alpha$, respectively, then $(s_{\alpha \rightarrow \beta}t_\alpha)_\beta$, called \textit{application}, is an HOL term of type $\beta$.
If $X_\alpha \in V_\alpha$ is a typed variable symbol and $s_\beta$ is an HOL term of type $\beta$, then $(\lambda X_\alpha s_\beta)_{\alpha \rightarrow \beta}$, called \textit{abstraction}, is an HOL term of type $\alpha \rightarrow \beta$.
The type of each term is given as a subscript (type subscripts, however, are often omitted if they are obvious in context).
We call terms of type $o$ \textit{formulas}.\footnote{HOL formulas should not be confused with the PAL formulas to be defined in \S\ref{sec:pal}; PAL formulas will later be identified in \S\ref{sec:embedding} with HOL predicates of type $(i\rightarrow o) \rightarrow i\rightarrow o$.}
As \textit{primitive logical connectives} we choose $\neg_{o \rightarrow o}, \vee_{o\rightarrow o\rightarrow o}$, $=_{\alpha \rightarrow \alpha \rightarrow \alpha }$ and $\Pi_{(\alpha \rightarrow o) \rightarrow o}$.
Other logical connectives can be introduced as abbreviations; e.g. $\longrightarrow_{o\rightarrow o\rightarrow o} = \lambda X_o \lambda Y_o \neg X \vee Y$.

\paragraph{Semantics of HOL.}
A \textit{frame} $\mathcal{D}$ for HOL is a collection $\{\mathcal{D}_\alpha\}_{\alpha \in T}$ of nonempty sets $\mathcal{D}_\alpha$, such that $\mathcal{D}_o = \{T, F\}$ (for true and false).
$\mathcal{D}_i$ is chosen freely and $\mathcal{D}_{\alpha \rightarrow \beta}$ are collections of functions mapping $\mathcal{D}_\alpha$ into $\mathcal{D}_\beta$.

A \textit{model} for HOL is a tuple $\mathcal{M} = \langle \mathcal{D}, I \rangle$, where $\mathcal{D}$ is a frame, and $I$ is a family of typed interpretation functions mapping constant symbols $p_\alpha \in C_\alpha$ to appropriate elements of $\mathcal{D}_\alpha$, called the \textit{denotation} of $p_\alpha$.
The logical connectives $\neg, \vee, \Pi$ and $=$ are always given their expected standard denotations:
\[\begin{tabular}{lll}
	$I(\neg_{o\rightarrow o})$ & = \textit{not} $\in \mathcal{D}_{o \rightarrow o}$ &s.t. \textit{not}(T) = F and \textit{not}(F) = T\\
	$I(\vee_{o \rightarrow o \rightarrow o})$ & = \textit{or} $\in \mathcal{D}_{o \rightarrow o \rightarrow o}$ &s.t. \textit{or}(a,b) = T iff (a = T or b = T) \\
	$I(=_{\alpha \rightarrow \alpha \rightarrow o})$ & = \textit{id} $\in \mathcal{D}_{\alpha \rightarrow \alpha \rightarrow o}$ &s.t. for all a,b $\in \mathcal{D}_\alpha$, \textit{id}(a,b) = T\\\
	&&\quad iff a is identical to b\\
	$I(\Pi_{(\alpha \rightarrow o) \rightarrow o})$ & = \textit{all} $\in \mathcal{D}_{(\alpha \rightarrow o) \rightarrow o}$ & s.t. for all $s \in \mathcal{D}_{\alpha \rightarrow o}$, \textit{all}(s) = T\\
	&&\quad iff s(a) = T for all a $\in \mathcal{D}_\alpha$ 
\end{tabular}\]
A \textit{variable assignment g} maps variables $X_\alpha$ to elements in $\mathcal{D}_\alpha$.
$g[d/W]$ denotes the assignment that is identical to $g$, except for variable $W$, which is now mapped to $d$.
	
The \textit{denotation} $\llbracket s_\alpha\rrbracket^{M,g}$ of an HOL term $s_\alpha$ on a model $\mathcal{M} = \langle \mathcal{D}, I \rangle$ under assignment $g$ is an element $d \in \mathcal{D}_\alpha$ defined in the following way:

\[
	\begin{tabular}{lll}
		$\llbracket p_\alpha\rrbracket^{\mathcal{M}, g}$ & = & $I(p_\alpha)$\\
		$\llbracket X_\alpha\rrbracket^{\mathcal{M}, g}$ & = & $g(X_\alpha)$\\
		$\llbracket (s_{\alpha \rightarrow \beta}t_\alpha)_\beta\rrbracket^{\mathcal{M}, g}$ & = & $\llbracket s_{\alpha \rightarrow \beta}\rrbracket^{\mathcal{M}, g}(\llbracket t_\alpha\rrbracket^{\mathcal{M}, g})$\\
		$\llbracket (\lambda X_\alpha s_\beta)_{\alpha \rightarrow \beta}\rrbracket^{\mathcal{M}, g}$ & = & the function $f$ from $\mathcal{D}_\alpha$ to $\mathcal{D}_\beta$\\
		&&\quad s.t. $f(d) = \llbracket s_\beta \rrbracket^{\mathcal{M},g[d/X_\alpha]}$ for all $d \in \mathcal{D}_\alpha$
	\end{tabular}
\]

In a \textit{standard model} a domain $\mathcal{D}_{\alpha \rightarrow \beta}$ is defined as the set of all total functions from $\mathcal{D}_\alpha$ to $\mathcal{D}_\beta$, i.e. $\mathcal{D}_{\alpha \rightarrow \beta} = \{ f \ |\ f : \mathcal{D}_\alpha \rightarrow \mathcal{D}_\beta \}$. 
In a \textit{Henkin model} (or general model)~\cite{henkin1950} function spaces are not necessarily required to be the full set of functions: $\mathcal{D}_{\alpha \rightarrow \beta} \subseteq \{ f \ |\ f : \mathcal{D}_\alpha \rightarrow \mathcal{D}_\beta \}$.  However, we require that the valuation function remains total, so that every term denotes. %
	
A HOL formula $s_o$ is \textit{true in a Henkin model $\mathcal{M}$ under assignment $g$} if and only if $\llbracket s_o\rrbracket^{\mathcal{M},g}=T$; also denoted by $\mathcal{M},g \models^\texttt{HOL} s_o$.
A HOL formula $s_o$ is called \textit{valid in $\mathcal{M}$}, denoted by $\mathcal{M} \models^\texttt{HOL} s_o$, iff $\mathcal{M},g \models^\texttt{HOL}s_o$ for all assignments $g$.
Moreover, a formula $s_o$ is called \textit{valid}, denoted by $\models^\texttt{HOL}s_o$, if and only if $s_o$ is valid in all Henkin models $\mathcal{M}$.

Due to G\"odel \cite{goedel1931} a sound and complete mechanization of HOL with standard semantics cannot be achieved. For HOL with Henkin semantics sound and complete calculi exist; cf.~e.g.~\cite{J6,B5} and the references therein.

Each standard model is obviously also a Henkin model.
Consequently, when a HOL formula is Henkin-valid, it is also valid in all standard models.

\section{Public Announcement Logic} \label{sec:pal}

The most important concepts and definitions of a \textit{public announcement logic (PAL) with relativized common knowledge} are depicted.
For more details, we refer to the literature \cite{van2007dynamic, pacuit2013dynamic}.

Before exploring these definitions, some general descriptions of the modeling approach are in order.
We use a graph-theoretical structure, called \textit{epistemic models}, to represent knowledge.
Epistemic models describe situations in terms of possible worlds.
A world represents one possibility about how the current situation can be.
At each world an agent considers all other reachable worlds (within the given equivalence class) as possible.
Each world in this set of possibilities needs to be consistent with the information the agent has.
Knowledge is described using an (binary) accessibility relation between worlds, rather than directly representing the agent's information.
A relation between worlds expresses that the agent is unable to tell which world is the one that represents the ``real'' situation.

Let $\mathcal{A}$ be a set of agents and $\mathcal{P}$ a set of atomic propositions.
Atomic propositions are intended to describe ground facts.
We use a set $W$ to denote possible worlds and a valuation function $V:\mathcal{P}\rightarrow \wp (W)$ that assigns a set of worlds to each atomic proposition. Vice versa, we may identify each world with the set of propositions that are true in them.

\begin{definition}[\textbf{Epistemic Model}]
	Let $\mathcal{A}$ be a (finite) set of agents and $\mathcal{P}$ a (finite or countable) set of atomic propositions.
	An \textit{epistemic model} is a triple	$\mathcal{M} = \langle W, \{ R_i \}_{i \in \mathcal{A}}, V \rangle$ where $W \not = \emptyset,\ R_i \subseteq W \times W$ is an accessibility relation (for each $i \in \mathcal{A}$), and $V: \mathcal{P} \rightarrow \wp(W)$ is a valuation function ($\wp(W)$ is the powerset of $W$).
\end{definition}

\noindent
Information of agent $i$ at world $w$ can now be defined as: $R_i(w) = \{v \in W\ | \ wR_iv \}$.
Having a separate (accessibility) relation for each agent enables them to have their own viewpoints.

Next, we introduce the syntax of our base epistemic logic as the set of sentences generated by the following grammar (where $p \in \mathcal{P}$ and $i \in \mathcal{A}$):
\[ \varphi := p \ | \ \neg \varphi \ | \ \varphi \vee \varphi \ | \ K_i \varphi \]

\noindent
We also introduce the abbreviations $\varphi \wedge \psi := \neg(\neg \varphi \vee \neg \psi)$ and $\varphi \rightarrow \psi := \neg \varphi \vee \psi$.

\begin{definition}[\textbf{Truth at world $w$}] \sloppy
	Given an epistemic model $\mathcal{M} = \langle W, \{ R_i \}_{i \in \mathcal{A}}, V \rangle$.
	For each $w \in W$, $\varphi$ \emph{is true at world $w$}, denoted $\mathcal{M}, w \models \varphi$, is defined inductively as follows:

    \vspace{1em}
    \begin{tabular}{lll}
    	$\mathcal{M},w \models p$ & iff & $w \in V(p)$\\
    	$\mathcal{M},w \models \neg \varphi$ & iff & $\mathcal{M}, w  \not \models \varphi$\\
    	$\mathcal{M},w \models \varphi \vee \psi$ & iff & $\mathcal{M},w \models \varphi$ or $\mathcal{M},w \models \psi$\\
    	$\mathcal{M},w \models K_i \varphi$ & iff & for all $v \in W$, if $wR_iv$ then $\mathcal{M},v \models \varphi$
    \end{tabular}
\end{definition}

\noindent	
The formula $K_i\varphi$  expresses that ``Agent $i$ knows $\varphi$''.
This describes knowledge as an all-or-nothing definition.
If we postulate that agent $i$ knows $\varphi$, we say that $\varphi$ is true throughout all worlds in agents $i$'s range of considerations (modeled as reachable worlds).

\textit{Truth} of a formula $\varphi$ for a model $\mathcal{M} = \langle W, \{ R_i \}_{i \in \mathcal{A}}, V \rangle$  and a world $w\in W$ is expressed by writing that $\mathcal{M},w \models \varphi$.
We define $V^\mathcal{M}(\varphi) = \{w \in W \ | \ \mathcal{M},w \models \varphi \}$. Formula $\varphi$ is \textit{valid} if and only if for all $\mathcal{M}$ and for all worlds $w$ we have $\mathcal{M},w \models \varphi$.

Our (multi-)modal logic above -- normal (multi-)modal logic K -- is not yet sufficiently suited to encode epistemic reasoning. Therefore, additional conditions (reflexivity, transitivity and euclideaness) are imposed on the accessibility relations. 
In the remainder of this article we therefore assume the validity of the following S5 principles for the agent's accessibility relations.\footnote{In the \logikey\ approach this is achieved by simply postulating the listed semantical properties.
Alternatively, the syntactic axiom schemata can be postulated in \logikey, and it is also possible to automatically prove the correspondences between them \cite{J21}. Of course, the work we present in this article can be easily adapted to support also  weaker versions of PAL.}

\begin{center}
	\begin{tabular}{llll}
		& \textbf{assumptions} & \textbf{axiom schemata} & \textbf{semantical properties}\\
		\hline
		\textbf{T} & \textit{truth} & $K_i \varphi \rightarrow \varphi$ & reflexive\\
		\textbf{4} & \textit{positive introspection} & $K_i \varphi \rightarrow K_i K_i \varphi$ & transitive\\
		\textbf{5} & \textit{negative introspection} & $\neg K_i \rightarrow K_i \neg K_i \varphi$ & euclidean
	\end{tabular}
\end{center}

We add public announcements \cite{plaza2007logics} to our logic.
The objective is to formulate an operation that informs all agents that some sentence $\varphi$ is true.
By the announcement, all agents will discard all worlds in which $\varphi$ is false, and consider only those worlds in which $\varphi$ is true.
Because of the \textit{publicity} of the announcement all agents are aware of the fact that all other agents know that $\varphi$ is true.

\begin{definition}[\textbf{Public Announcement}]
    \sloppy 
	Suppose that $\mathcal{M} = \langle W, \{R_i\}_{i\in\mathcal{A}}, V \rangle$ is an epistemic model and $\varphi$ is a formula (in the language of our base logic).
	After $\varphi$ is publicly announced, the resulting model is
	$\mathcal{M}^{!\varphi} = \langle W^{!\varphi}, \{R_i^{!\varphi}\}_{i\in\mathcal{A}}, V^{!\varphi} \rangle$ where $W^{!\varphi} = \{ w \in W \ | \ \mathcal{M},w \models \varphi \}$, $R_i^{!\varphi} = R_i \cap (W^{!\varphi} \times W^{!\varphi})$ for all $i \in \mathcal{A}$, and $V^{!\varphi}(p) = V(p) \cap W^{!\varphi}$ for all $p \in \mathcal{P}$.
	
	To say that ``\textit{$\psi$ is true after the announcement of $\varphi$}'' is represented as $[!\varphi]\psi$.
	Truth for this new operator at world $w$ in $\mathcal{M}$ is defined as:
	\[ \mathcal{M}, w \models [!\varphi]\psi \text{ iff } \mathcal{M},w \not\models \varphi \textrm{ or } \mathcal{M}^{!\varphi}, w \models \psi \]
\end{definition}

We conclude this section with the introduction of notions for group knowledge.

Mutual knowledge, often stated as \textit{everyone knows}, describes knowledge that each member of the group holds.
Usually, it is defined for a group of agents $G \subseteq \mathcal{A}$ as $E_G \varphi := \bigwedge_{i \in G}K_i \varphi$.
Equivalently, a new relation can be introduced to express mutual knowledge with the knowledge operator.

\begin{definition}[\textbf{Mutual Knowledge}]
    Let $G \subseteq \mathcal{A}$ be  a group of agents.
	Let $R_G = \bigcup_{i \in G}R_i$.
	The truth clause for mutual knowledge is:
	\[
		\mathcal{M},w \models E_G \psi \text{ iff for all $v \in W$, if } wR_Gv \text{ then } \mathcal{M}, v \models \psi
	\]
\end{definition}

To describe knowledge that is obtained when all agents put their individual knowledge together we introduce \textit{distributed knowledge}.
\begin{definition}[\textbf{Distributed Knowledge}]
    Let $G \subseteq \mathcal{A}$ be  a group of agents.
	Let $R_D = \bigcap_{i \in G}R_i$.
	The truth clause for mutual knowledge is:
	\[
		\mathcal{M},w \models D_G \psi \text{ iff for all $v \in W$, if } wR_Dv \text{ then } \mathcal{M}, v \models \psi
	\]
\end{definition}

Still, there is a distinction to make between \textit{everyone knows} $\varphi$ and \textit{it is common knowledge that} $\varphi$.
A statement $p$ is common knowledge when all agents know $p$, know that they all know $p$, know that they all know that they all know $p$, and so ad infinitum.
Relativized common knowledge was introduced by van Benthem, van Eijck and Kooi \cite{benthem2006} as a variant of common knowledge.
As the name suggests, knowledge update is then treated as a \textit{relativization}.

\begin{definition}[\textbf{Relativized Common Knowledge}]
	Let $G \subseteq \mathcal{A}$ be  a group of agents.
	Let $R_G = \bigcup_{i \in G}R_i$.
	The truth clause for relativized common knowledge is:
	\vspace{-.25em}
	\[\mathcal{M},w \models \mathcal{C}_G(\varphi | \psi) \text{ iff for all $v \in W$, if } w(R_G^\varphi)^+v \text{ then } \mathcal{M}, v \models \psi\]
	where $R_G^\varphi = R_G \cap (W \times $$V^{\mathcal{M}}(\varphi)$), and $(R_G^\varphi)^+$ denotes the transitive closure of $R_G^\varphi$.
\end{definition}

Intuitively, $\mathcal{C}_G(\varphi | \psi)$ expresses that $\psi$ is common knowledge among the agents in group $G$ relative to the information that $\varphi$ is true.
This means, that every path from $w$, that is accessible using the agent's relations through worlds in which $\varphi$ is true, must end in a world in which $\psi$ is true.
Ordinary unconditional common knowledge of $\varphi$ can be abbreviated as $\mathcal{C}_G(\top | \varphi)$, where $\top$ denotes an arbitrary tautology.

In the remainder we use PAL to refer to the depicted logic consisting of modal logic K, extended by the principles T45, public announcement and relativized common knowledge.
The logic PAL, which we employ for the modeling of the wise men puzzle in the remainder, is now given as follows (where $p \in \mathcal{P}$, $i \in \mathcal{A}$, and $G \subseteq \mathcal{A}$):\footnote{The logic can be extended to include mutual knowledge ($E_G$) and distributed knowledge ($D_G$) if needed, and we also cover their embeddings in the remainder.}
\[ \varphi := p \ | \ \neg \varphi \ | \ \varphi \vee \varphi \ | \ K_i \varphi \ | \ \mathcal{C}_G(\varphi | \varphi) \ | \ [!\varphi]\varphi \]

\section{Modeling PAL as a Fragment of HOL}

A shallow semantical embedding (SSE) of a target logic into HOL provides a translation between  the two logics in such a way that the former logic is identified and characterized as a proper fragment of the latter.\footnote{The SSE technique is not be confused with higher-order abstract syntax \cite{DBLP:conf/pldi/PfenningE88}, or with other forms of deep embeddings.}
Once such an SSE is obtained, all that is needed to prove (or refute) conjectures in the target logic is to provide the SSE, encoded in an input file, to the HOL prover in addition to the encoded conjecture.
We can then use the HOL prover as-is, without making any changes to its source code, and use it to solve problems in our target logic.

\subsection{Shallow Semantical Embedding} \label{sec:embedding}
	
To define an SSE for target logic PAL, we lift the type of propositions in order to explicitly encode their dependency on possible worlds; this is analogous to prior work \cite{J41,J44}.
In order to capture the model-changing behavior of PAL, we additionally introduce world domains (sets of worlds) as parameters/arguments in the encoding.
The rationale thereby is to suitably constrain, and recursively pass-on, these domains after each model changing action.

PAL formulas are thus identified in our semantical embedding with certain HOL terms (predicates) of type \mbox{$(i \rightarrow o) \rightarrow i \rightarrow o$}. 
They can be applied to terms of type  \mbox{$i\rightarrow o$}, which are assumed to denote evaluation domains, and subsequently to terms of type $i$, which are assumed to denote possible worlds.
That is, the HOL type $i$ is identified with a (non-empty) set of worlds, and the type \mbox{$i \rightarrow o$}, abbreviated by $\sigma$, is identified with a set of sets of worlds, i.e., a set of evaluation domains.
	
Type \mbox{$(i \rightarrow o) \rightarrow i \rightarrow o$} is abbreviated as $\tau$, $\alpha$ is an abbreviation for \mbox{$i\rightarrow i\rightarrow o$}, the type of accessibility relations between worlds, and $\varrho$ abbreviates \mbox{$\alpha \rightarrow o$}, the type of sets of accessibility relations.
Table 1 provides an overview of these abbreviations.
	
For each propositional symbol $p^i$ of PAL, the associated HOL signature is assumed to contain a corresponding constant symbol $p^i_\sigma$, which is (rigidly) denoting the set of all those worlds in which $p^i$ holds.
We call the $p^i_\sigma$ \textit{$\sigma$-type-lifted propositions}. 
Moreover, for $k = 1, \dots , |\mathcal{A}|$ the HOL signature is assumed to contain the constant symbols $r^1_\alpha, \dots, r^{|\mathcal{A}|}_\alpha$. 
Without loss of generality, we assume that besides those constants symbols and the primitive logical connectives of HOL, no other constant symbols are given in the signature of HOL.

\begin{table}
    \centering
    \begin{tabular}{lll}
        \textbf{meaning} & \textbf{type} & \textbf{abbreviations}\\
        \hline
        HOL formula & $o$ & \\
        worlds & $i$ & \\
        evaluation domains & $i \rightarrow o$ & $\sigma$ \\
        PAL formulas & $(i \rightarrow o) \rightarrow i \rightarrow o$ & $\tau$ \\
        accessibility relations & $i \rightarrow i \rightarrow o$ & $\alpha$ \\
        set of relations & $\alpha \rightarrow o$ & $\varrho$
    \end{tabular}
    \label{table:types}
    \caption{Type abbreviations as used in the remainder}
\end{table}

\begin{definition}[\textbf{Mapping of PAL formulas $\varphi$ into HOL terms $\lfloor\varphi\rfloor$}]
    The mapping $\lfloor \cdot \rfloor$ translates a formula $\varphi$ of PAL into a term $\lfloor \varphi \rfloor$ of HOL of type $\tau$.
    The mapping is defined recursively:
    \begin{align*}
        \lfloor p^j \rfloor &= (^A(p^j_\sigma))_\tau\\
    	\lfloor \neg \varphi \rfloor &= \neg_{\tau\rightarrow\tau} \lfloor \varphi \rfloor\\
    	\lfloor \varphi \vee \psi \rfloor &= \vee_{\tau\rightarrow\tau\rightarrow\tau} \lfloor \varphi \rfloor \lfloor \psi \rfloor\\
    	\lfloor K \ \text{r}^k \ \varphi \rfloor &= K_{\alpha \rightarrow \tau \rightarrow \tau} \ \text{r}^k_{\alpha} \ \lfloor \varphi \rfloor\\
    	\lfloor [! \varphi ] \psi \rfloor &= [! \ \cdot \ ] \cdot_{\tau \rightarrow \tau \rightarrow \tau} \lfloor \varphi \rfloor \lfloor \psi \rfloor\\
    	\lfloor \mathcal{C}_G( \varphi | \psi ) \rfloor &= \mathcal{C}_\cdot( \cdot | \cdot)_{\varrho \rightarrow \tau \rightarrow \tau \rightarrow \tau} G_\varrho \lfloor \varphi \rfloor \lfloor \psi \rfloor
    \end{align*}
    
    A recursive definition is actually not needed in practice. By inspecting the above equations, it becomes clear that only the abbreviations for the logical connectives of PAL are required in combination with a type-lifting for the propositional symbols; cf.~the non-recursive equations of the actual encoding in Lines 28-36 and 46-47 of Fig.~\ref{fig:bild1} in Appendix A.
    	
    Operator $^A(\cdot)$, which evaluates atomic formulas, is defined as follows:
    \begin{align*}
        ^A\cdot_{\sigma\rightarrow\tau} &= \lambda A_\sigma \lambda D_\sigma \lambda X_i (D \ X \wedge A\ X)
    \end{align*}
\end{definition}
		
As a first argument, it accepts a $\sigma$-type-lifted proposition $A_\sigma$, which are rigidly interpreted.
As a second argument, it accepts an evaluation domain $D_\sigma$, that is, an arbitrary subset of the domain associated with type $\sigma$.
And as a third argument, it accepts a current world $X_i$.
It then checks whether (i) the current world is a member of evaluation domain $D_\sigma$ and (ii) whether the $\sigma$-type-lifted proposition $A_\sigma$ holds in the current world.
		
The other logical connectives of PAL, except for $[! \ \cdot \ ] \cdot_{\tau \rightarrow \tau \rightarrow \tau}$,  are now defined in a way so that they simply pass on the evaluation domains as parameters to the atomic-level.
Only $[! \ \cdot \ ] \cdot_{\tau \rightarrow \tau \rightarrow \tau}$ is modifying, in fact, constraining, the evaluation domain it passes on, and it does this in the expected way (cf. Def.~3.3):
\begin{align*}
	\neg_{\tau\rightarrow\tau} &= \lambda A_\tau \lambda D_\sigma \lambda X_i \neg (A \ D \ X) \\ \vee_{\tau\rightarrow\tau\rightarrow\tau} &= \lambda A_\tau \lambda B_\tau \lambda D_\sigma \lambda X_i (A \ D \ X \vee B \ D \ X) \\
	K_{\alpha \rightarrow \tau \rightarrow \tau} &= \lambda R_\alpha \lambda A_\tau \lambda D_\sigma \lambda X_i \forall Y_i ((D\ Y\ \wedge\ R \ X \ Y)\ \longrightarrow A \ D \ Y)\\
	[! \ \cdot \ ] \cdot_{\tau \rightarrow \tau \rightarrow \tau} &= \lambda A_\tau \lambda B_\tau \lambda D_\sigma \lambda X_i ((A \ D \ X) \longrightarrow (B \ (\lambda Y_i\ (D\ Y\ \wedge \ A\ D\ Y))\ X))
\end{align*}

To model $\mathcal{C}_\cdot( \cdot | \cdot)_{\varrho \rightarrow \tau \rightarrow \tau \rightarrow \tau}$ we reuse the following operations on relations; cf.~\cite{J41,J44}.
\begin{align*}
    \texttt{transitive}_{\alpha\rightarrow o} &= \lambda R_\alpha \forall X_i \forall Y_i \forall Z_i ((R\ X\ Y\ \wedge\ R\ Y\ Z)\ \longrightarrow\ R\ X\ Z)\\
    \texttt{intersection}_{\alpha\rightarrow\alpha\rightarrow\alpha} &= \lambda R_\alpha \lambda Q_\alpha \lambda X_i \lambda Y_i (R\ X\ Y\ \wedge\ Q\ X\ Y)\\
    \texttt{sub}_{\alpha\rightarrow\alpha\rightarrow o} &= \lambda R_\alpha \lambda Q_\alpha \forall X_i \forall Y_i (R\ X\ Y\ \longrightarrow\ Q\ X\ Y)
\end{align*}
	
The transitive closure $\texttt{tc}$ of a relation can be elegantly defined in HOL as:
\begin{align*}
    \texttt{tc}_{\alpha\rightarrow\alpha} &= \lambda R_\alpha \lambda X_i \lambda Y_i \forall Q_\alpha (\texttt{transitive}\ Q\ \longrightarrow\ (\texttt{sub}\ R\ Q\ \longrightarrow\ Q\ X\ Y))
\end{align*}

Additionally, we introduce higher-order definitions for the union and intersection of an arbitrary set of relations.
\begin{align*}
    \texttt{big\_union}_{\varrho\rightarrow\alpha} &= \lambda G_{\varrho} \lambda X_i \lambda Y_i \exists R_\alpha (G\ R\ \wedge\ R\ X\ Y)\\
    \texttt{big\_intersection}_{\varrho\rightarrow\alpha} &= \lambda G_{\varrho} \lambda X_i \lambda Y_i \forall R_\alpha (G\ R\ \longrightarrow\ R\ X\ Y)
\end{align*}

\texttt{EVR}, being applied to a set of accessibility relations $G$, returns a relation denoting the mutual knowledge of the set/group of agents $G$. \texttt{EVR} is subsequently used in the definition of relativized common knowledge to describe the mutual knowledge of multiple agents.
Analogously, we introduce distributed knowledge \texttt{DIS} as the intersection of this set.
In the case of three agents, we introduce a concrete set of relations of $R$ consisting of $r^1$, $r^2$ and $r^3$ of type $\alpha$.
\begin{align*}
    \texttt{G}_\varrho &= \lambda R_\alpha (R = r^1\ \vee\ R = r^2\ \vee\ R = r^3)\\
    \texttt{EVR}_\varrho &= \lambda G_\varrho (\texttt{big\_union}\ G)\\
    \texttt{DIS}_\varrho &= \lambda G_\varrho (\texttt{big\_intersection}\ G)
\end{align*}

One could also use a less verbose way of defining $\texttt{G}$ and leverage Isabelle's set notation (and associated theory): $\texttt{G}_{\varrho} = \{ r^1, r^2, r^3 \}$.
However, we here deliberately choose a direct encoding in HOL in order to introduce as little unnecessary dependencies as possible.

The operator $\mathcal{C}_\cdot( \cdot | \cdot)_{\varrho \rightarrow \tau \rightarrow \tau \rightarrow \tau}$ thus abbreviates the following HOL term:
\begin{align*}
    \mathcal{C}_\cdot( \cdot | \cdot)_{\varrho \rightarrow \tau \rightarrow \tau \rightarrow \tau} &= \lambda G_\varrho \lambda A_\tau \lambda B_\tau \lambda D_\sigma \lambda X_i \forall Y_i\\
	&\hspace*{1.5em}(\texttt{tc}\ (\texttt{intersection} \ (\texttt{EVR}\ G)\ (\lambda U_i \lambda V_i (D\ V\ \wedge\ A\ D\ V)))\ X\ Y\\
	&\hspace*{1.75em}\longrightarrow \ B\ D\ Y)   
\end{align*}

Analyzing the truth of a PAL formula $\varphi$, represented by the HOL term $\lfloor \varphi \rfloor$, in a particular domain $d$, represented by the term $D_\sigma$, and a world $s$, represented by the term $S_i$, corresponds to evaluating the application ($\lfloor \varphi \rfloor \ D_\sigma \ S_i$).
We can verify whether $S$ denotes a world in $D$ by checking if $(D\ S)$ is true.
If that is the case, we evaluate $\varphi$ for this domain and world.

A formula $\varphi$ is thus generally valid if and only if for all  $D_\sigma$ and all  $S_i$ we have $D\ S \rightarrow \lfloor \varphi \rfloor D\ S$.
The validity function, therefore, is defined as follows:
\[ \texttt{vld}_{\tau \rightarrow o} = \lambda A_\tau \forall D_\sigma \forall S_i (D\ S\ \longrightarrow\ A \ D\ S). \]

The necessity to quantify over all possible domains in this definition will be further illustrated below.
	
\subsection{Encoding into Isabelle/HOL}	\label{sec:encoding}

What follows is a description of the concrete encoding of the presented SSE of PAL in HOL within the higher-order proof assistant Isabelle/HOL (cf.~Fig.~\ref{fig:bild1} in App.~A).\footnote{The full sources of our encoding can be found at \url{http://logikey.org} in subfolder \texttt{Public-Announcement-Logic}.}

All necessary types can be modeled in a straightforward way.
We declare \texttt{i} to denote possible worlds and then introduce type aliases for $\sigma$, $\tau$, $\alpha$ and $\varrho$.
Type \texttt{bool} represents (the bivalent set of) truth values introduced as $o$ before.

\vspace{.5em}
\lstinline[language=Isabelle]|typedecl i|\hspace*{9.7em}\lstinline|(* Type of possible worlds *)|\\
\hspace*{1.5em}\lstinline[language=Isabelle]|type_synonym $\sigma$ = "i$\Rightarrow$bool"|\hspace*{1.9em}\lstinline|(* Type of world domains *)|\\
\hspace*{1.5em}\lstinline[language=Isabelle]|type_synonym $\tau$ = "$\sigma$$\Rightarrow$i$\Rightarrow$bool" (* Type of world depended formulas *)|\\
\hspace*{1.5em}\lstinline[language=Isabelle]|type_synonym $\alpha$ = "i$\Rightarrow$i$\Rightarrow$bool"|\hspace*{.5em}\lstinline|(* Type of accessibility relations *)|\\
\hspace*{1.5em}\lstinline[language=Isabelle, basicstyle=\small\ttfamily]|type_synonym $\varrho$ = "$\alpha$$\Rightarrow$bool"|\hspace*{1.9em}\lstinline|(* Type of group of agents *)|

\vspace{.5em}
The agents are declared as accessibility relations, and the group of agents can then be denoted by a predicate of type $\varrho$. 
In order to obtain $\mathcal{S}5$ (KT45) properties, we declare respective conditions on the accessibility relations in the group of agents $\texttt{A}$.
Various Isabelle/HOL encodings from \cite{J41,J44} are reused here, including the encoding of transitive closure.

\vspace{.5em}
\lstinline[language=Isabelle, basicstyle=\small\ttfamily]|abbreviation S5Agent::"$\alpha$$\Rightarrow$bool"|\\
\hspace*{3em}\lstinline[language=Isabelle, basicstyle=\small\ttfamily]|where "S5Agent i $\equiv$ reflexive i $\wedge$ transitive i $\wedge$ euclidean i"|\\
\hspace*{1.5em}\lstinline[language=Isabelle, basicstyle=\small\ttfamily]|abbreviation S5Agents::"$\varrho$$\Rightarrow$bool"|\\
\hspace*{3em}\lstinline[language=Isabelle, basicstyle=\small\ttfamily]|where "S5Agents A $\equiv$ $\forall$i. (A i $\longrightarrow$ S5Agent i)"|
\vspace{.5em}

Each of the lifted unary and binary connectives of PAL accepts arguments of type $\tau$, i.e.~lifted PAL formulas, and returns such a lifted PAL formula.

A special case, as discussed before, is the new operator for atomic propositions $^A(\cdot)$.
When evaluating $\sigma$-type lifted atomic propositions $p$ we need to check if $p$ is true in the given world \texttt{w}, but we also need to check whether the given world \texttt{w} is still part of our evaluation domain \texttt{W} that has been recursively passed on.
Operator $^A(\cdot)$ is thus of type \lstinline[language=Isabelle]|"$\sigma \Rightarrow \tau$"|. 
The need and effect of this distinction are addressed again in \S\ref{sec:failures}, where we present formulas that are only valid when restricted to the atomic case.

\vspace{.5em}
\lstinline[language=Isabelle]|abbreviation patom::"$\sigma\Rightarrow\tau$" ("$^\texttt{A}$_")|\\
\hspace*{3em}\lstinline[language=Isabelle]|where "$^\texttt{A}$p $\equiv$ $\lambda$W w. W w $\wedge$ p w"|\\
\hspace*{1.5em}\lstinline[language=Isabelle]|abbreviation ptop::"$\tau$" ("$\boldsymbol{\top}$")|\\
\hspace*{3em}\lstinline[language=Isabelle]|where "$\boldsymbol{\top}$ $\equiv$ $\lambda$W w. True"|\\
\hspace*{1.5em}\lstinline[language=Isabelle]|abbreviation pneg::"$\tau$$\Rightarrow$$\tau$" ("$\boldsymbol{\neg}$")|\\
\hspace*{3em}\lstinline[language=Isabelle]|where "$\boldsymbol{\neg} \varphi$ $\equiv$ $\lambda$W w. $\neg$($\varphi$ W w)"|\\
\hspace*{1.5em}\lstinline[language=Isabelle]|abbreviation pand::"$\tau$$\Rightarrow$$\tau$$\Rightarrow$$\tau$" ("$\boldsymbol{\wedge}$")|\\
\hspace*{3em}\lstinline[language=Isabelle]|where "$\varphi \boldsymbol{\wedge} \psi$ $\equiv$ $\lambda$W w. ($\varphi$ W w) $\wedge$ ($\psi$ W w)"|\\
\hspace*{1.5em}\lstinline[language=Isabelle]|abbreviation por::"$\tau$$\Rightarrow$$\tau$$\Rightarrow$$\tau$" ("$\boldsymbol{\vee}$")|\\
\hspace*{3em}\lstinline[language=Isabelle]|where "$\varphi \boldsymbol{\vee} \psi$ $\equiv$ 	$\lambda$W w. ($\varphi$ W w) $\vee$ ($\psi$ W w)"|\\
\hspace*{1.5em}\lstinline[language=Isabelle]|abbreviation pimp::"$\tau$$\Rightarrow$$\tau$$\Rightarrow$$\tau$" ("$\boldsymbol{\rightarrow}$")|\\
\hspace*{3em}\lstinline[language=Isabelle]|where "$\varphi \boldsymbol{\rightarrow} \psi$ $\equiv$ 	$\lambda$W w. ($\varphi$ W w) $\longrightarrow$ ($\psi$ W w)"|\\	
\hspace*{1.5em}\lstinline[language=Isabelle]|abbreviation pequ::"$\tau$$\Rightarrow$$\tau$$\Rightarrow$$\tau$" ("$\boldsymbol{\leftrightarrow}$")|\\
\hspace*{3em}\lstinline[language=Isabelle]|where "$\varphi \boldsymbol{\leftrightarrow} \psi$ $\equiv$ 	$\lambda$W w. ($\varphi$ W w) $\longleftrightarrow$ ($\psi$ W w)"|

\vspace{.5em}
In the definition of the knowledge operator \texttt{K}, we have to make sure to add a domain check in the implication.

\vspace{.5em}
\lstinline[language=Isabelle]|abbreviation pknow::"$\alpha$$\Rightarrow$$\tau$$\Rightarrow$$\tau$" ("$\textbf{K}$_ _")|\\
\hspace*{3em}\lstinline[language=Isabelle]|where "$\textbf{K}$ r $\varphi$ $\equiv \lambda$W w.$\forall$v. (W v $\wedge$ r w v)  $\longrightarrow$  ($\varphi$ W v)"|

\vspace{.5em}
Some additional abbreviations are introduced to improve readability.
One is a more concise way to state knowledge, achieved by abbreviating \texttt{pknow} even further.

\vspace{.5em}
\lstinline[language=Isabelle, basicstyle=\small\ttfamily]|abbreviation agtknows::"$\alpha$$\Rightarrow$$\tau$$\Rightarrow$$\tau$" ("$\textbf{K}_{\_}$ _")|\\
\hspace*{3em}\lstinline[language=Isabelle]|where "$\textbf{K}_\texttt{r}\, \varphi$ $\equiv \textbf{K}$ r $\varphi$"|
\vspace{.5em}

Additionaly, operators for mutual and distributed knowledge are presented.
To achieve this we introduce two additional encodings on relations. The union and intersection operators on a set of relations.

\vspace{.5em}
\lstinline[language=Isabelle, basicstyle=\small\ttfamily]|definition big_union_rel::"$\varrho$$\Rightarrow$$\alpha$"|\\
\hspace*{3em}\lstinline[language=Isabelle, basicstyle=\small\ttfamily]|where "big_union_rel X $\equiv$ $\lambda$u v. $\exists$R. (X R) $\wedge$ (R u v)"|\\
\hspace*{1.5em}\lstinline[language=Isabelle, basicstyle=\small\ttfamily]|definition big_intersection_rel::"$\varrho$$\Rightarrow$$\alpha$"|\\
\hspace*{3em}\lstinline[language=Isabelle, basicstyle=\small\ttfamily]|where "big_intersection_rel X $\equiv$ $\lambda$u v. $\forall$R. (X R) $\longrightarrow$ (R u v)"|

\vspace{.5em}
These encodings can then be applied to concrete groups of agents $G$ of type $\varrho$, to define the relations (\texttt{EVR G}) and (\texttt{DIS G}).

\vspace{.5em}
\lstinline[language=Isabelle, basicstyle=\small\ttfamily]|abbreviation EVR::"$\varrho$$\Rightarrow$$\alpha$" where "EVR G $\equiv$ big_union_rel G"|\\
\hspace*{1.5em}\lstinline[language=Isabelle, basicstyle=\small\ttfamily]|abbreviation DIS::"$\varrho$$\Rightarrow$$\alpha$" where "DIS G $\equiv$ big_intersection_rel G|
\vspace{.5em}

Now it is straightforward to abbreviate mutual and distributed knowledge.

\vspace{.5em}
\lstinline[language=Isabelle, basicstyle=\small\ttfamily]|abbreviation evrknows::"$\varrho$$\Rightarrow$$\tau$$\Rightarrow$$\tau$" ("$\textbf{E}_{\_}$ _")|\\
\hspace*{3em}\lstinline[language=Isabelle, basicstyle=\small\ttfamily]|where "$\textbf{E}_\texttt{G}\, \varphi \equiv \textbf{K}$ (EVR G) $\varphi$"|\\
\hspace*{1.5em}\lstinline[language=Isabelle, basicstyle=\small\ttfamily]|abbreviation disknows::"$\varrho$$\Rightarrow$$\tau$$\Rightarrow$$\tau$" ("$\textbf{D}_{\_}$ _")|\\
\hspace*{3em}\lstinline[language=Isabelle, basicstyle=\small\ttfamily]|where "$\textbf{D}_\texttt{G}\, \varphi \equiv \textbf{K}$ (DIS G) $\varphi$"|

\vspace{.5em}
We finally see the change of the evaluation domain in action, when introducing the public announcement operator.
We already inserted domain checks in the definition of the operators \texttt{K} and $^A(\cdot)$.
Now, we need to constrain the domain after each public announcement.
So far the evaluation domain, modeled by \texttt{W}, got passed on through all lifted operators without any change.
In the public announcement operator, however, we modify the evaluation domain \texttt{W} into $(\lambda\texttt{z}.\ \texttt{W z}\ \wedge\ \varphi\ \texttt{W z})$ (i.e., the set of all worlds \texttt{z} in \texttt{W}, such that $\varphi$ holds for \texttt{W} and \texttt{z}),
which is then  recursively passed on. 
The public announcement operator is thus defined as:

\vspace{.5em}
\lstinline[language=Isabelle]|abbreviation ppal :: "$\tau$$\Rightarrow$$\tau$$\Rightarrow$$\tau$" ("$\boldsymbol{\text{[}}\boldsymbol{\text{!}}$_$\boldsymbol{\text{]}}$_")|\\
\hspace*{3em}\lstinline[language=Isabelle]|where "$\boldsymbol{\text{[}}\boldsymbol{\text{!}}\varphi\boldsymbol{\text{]}}\psi$ $\equiv$ $\lambda$W w. ($\varphi$ W w) $\longrightarrow$ ($\psi$ ($\lambda$z. W z $\wedge$ $\varphi$ W z) w)"|

\vspace{.5em}
The following embedding of relativized common knowledge is a straightforward encoding of the semantic properties and definitions as proposed in Def.~5.

\vspace{.5em}
\lstinline[language=Isabelle, basicstyle=\small\ttfamily]|abbreviation prck :: "$\varrho$$\Rightarrow$$\tau$$\Rightarrow$$\tau$$\Rightarrow\tau$" ("$\textbf{C}_\_\boldsymbol{\llparenthesis}$_$\boldsymbol{|}$_$\boldsymbol{\rrparenthesis}$")|\\
\hspace*{3em}\lstinline[language=Isabelle, basicstyle=\small\ttfamily]|where "$\textbf{C}_\texttt{G}\boldsymbol{\llparenthesis}\varphi\boldsymbol{|}\psi\boldsymbol{\rrparenthesis}$" $\equiv$ $\lambda$W w. $\forall$v.|\\
 \hspace*{4.5em}\lstinline[language=Isabelle, basicstyle=\small\ttfamily]|(tc (intersection_rel (EVR G) ($\lambda$u v. W v $\wedge\ \varphi$ W w)) w v) $\longrightarrow$ ($\psi$ W v)"|

\vspace{.5em}
As described earlier we can abbreviate ordinary common knowledge as $\mathcal{C}_G(\top | \varphi)$:

\vspace{.5em}
\lstinline[language=Isabelle, basicstyle=\small\ttfamily]|abbreviation pcmn :: "$\varrho$$\Rightarrow$$\tau$$\Rightarrow$$\tau$" ("$\textbf{C}_\_$ _")|\\ \hspace*{3em}\lstinline[language=Isabelle, basicstyle=\small\ttfamily]|where "$\textbf{C}_{\texttt{G}}\ \varphi$ $\equiv$ $\textbf{C}_\texttt{G}\boldsymbol{\llparenthesis}\boldsymbol{\top|}\varphi\boldsymbol{\rrparenthesis}$"|

\vspace{.5em}
Finally, an embedding for the notion of validity is needed.
Generally, for a type-lifted formula $\varphi$ to be valid, the application of $\varphi$ has to hold true for all worlds \texttt{w}.
In the context of PAL the evaluation domains also have to be incorporated in the definition.
Originally, we were tempted to define PAL validity in such a way that we start with a ``full evaluation domain'', a domain that evaluates to \texttt{True} for all possible worlds and gets restricted, whenever necessary after an announcement.
Such a validity definition would look like this:

\vspace{.5em}
\lstinline[language=Isabelle]|abbreviation tvalid::"$\tau$$\Rightarrow$bool" ("$\lfloor$_$\rfloor^\texttt{T}$")|\\
\hspace*{3em}\lstinline[language=Isabelle]|where "$\lfloor$_$\rfloor^\texttt{T}$ $\equiv$ $\forall$w. $\varphi$ ($\lambda$x. True) w"|

\vspace{.5em}
But this leads to undesired behavior, which we can easily see when using our reasoning tools to study e.g.~the validity \textit{announcement necessitation}: $\textit{from}\ \varphi\textit{, infer}\ [!\psi]\varphi$.
If we check for a counterexample in Isabelle/HOL, the model finder Nitpick \cite{Nitpick} reports the following:

\vspace{.5em}
\lstinline[language=Isabelle, basicstyle=\small\ttfamily]|lemma necessitation: assumes "$\lfloor \varphi \rfloor^\texttt{T}$" shows "$\lfloor\boldsymbol{[!}\psi\boldsymbol{]}\varphi\rfloor^\texttt{T}$" nitpick oops|\\
\noindent\rule{\textwidth}{0.4pt}

\begin{lstlisting}[basicstyle=\small\ttfamily]
Nitpick found a counterexample for card i = 2:

Free variables:
$\varphi$ = ($\lambda$x. _)
      ((($\lambda$x. _)($\texttt{i}_1$ := True, $\texttt{i}_2$ := True), $\texttt{i}_1$) := True,
\end{lstlisting}
\vspace{-1em}
\begin{lstlisting}[backgroundcolor=\color{gray!15}, basicstyle=\small\ttfamily]
       (($\lambda$x. _)($\texttt{i}_1$ := True, $\texttt{i}_2$ := True), $\texttt{i}_2$) := True,
\end{lstlisting}
\vspace{-1em}
\begin{lstlisting}[basicstyle=\small\ttfamily]
       (($\lambda$x. _)($\texttt{i}_1$ := True, $\texttt{i}_2$ := False), $\texttt{i}_1$) := False,
       (($\lambda$x. _)($\texttt{i}_1$ := True, $\texttt{i}_2$ := False), $\texttt{i}_2$) := False,
       (($\lambda$x. _)($\texttt{i}_1$ := False, $\texttt{i}_2$ := True), $\texttt{i}_1$) := False,
\end{lstlisting}
\vspace{-1em}
\begin{lstlisting}[backgroundcolor=\color{gray!15}, basicstyle=\small\ttfamily]
       (($\lambda$x. _)($\texttt{i}_1$ := False, $\texttt{i}_2$ := True), $\texttt{i}_2$) := False,
\end{lstlisting}
\vspace{-1em}
\begin{lstlisting}[basicstyle=\small\ttfamily]
       (($\lambda$x. _)($\texttt{i}_1$ := False, $\texttt{i}_2$ := False), $\texttt{i}_1$) := False,
       (($\lambda$x. _)($\texttt{i}_1$ := False, $\texttt{i}_2$ := False), $\texttt{i}_2$) := False)
$\psi$ = ($\lambda$x. _)
\end{lstlisting}
\vspace{-1em}
\begin{lstlisting}[backgroundcolor=\color{gray!15}, basicstyle=\small\ttfamily]
      ((($\lambda$x. _)($\texttt{i}_1$ := True, $\texttt{i}_2$ := True), $\texttt{i}_1$) := False,
       (($\lambda$x. _)($\texttt{i}_1$ := True, $\texttt{i}_2$ := True), $\texttt{i}_2$) := True,
\end{lstlisting}
\vspace{-1em}
\begin{lstlisting}[basicstyle=\small\ttfamily]
       (($\lambda$x. _)($\texttt{i}_1$ := True, $\texttt{i}_2$ := False), $\texttt{i}_1$) := False,
       (($\lambda$x. _)($\texttt{i}_1$ := True, $\texttt{i}_2$ := False), $\texttt{i}_2$) := False,
       (($\lambda$x. _)($\texttt{i}_1$ := False, $\texttt{i}_2$ := True), $\texttt{i}_1$) := False,
       (($\lambda$x. _)($\texttt{i}_1$ := False, $\texttt{i}_2$ := True), $\texttt{i}_2$) := False,
       (($\lambda$x. _)($\texttt{i}_1$ := False, $\texttt{i}_2$ := False), $\texttt{i}_1$) := False,
       (($\lambda$x. _)($\texttt{i}_1$ := False, $\texttt{i}_2$ := False), $\texttt{i}_2$) := False)
Skolem constant:
  w = $\texttt{i}_2$
\end{lstlisting}

Here, for world $\texttt{i}_2$ the term $\lfloor \varphi \rfloor^\texttt{T}$ is true. Because we evaluate
\[
    \texttt{(($\lambda$x.\ \_)($\texttt{i}_1$ := True, $\texttt{i}_2$ := True), $\texttt{i}_2$) := True}
\]

Yet, the evaluation of the term $\lfloor\boldsymbol{[!}\psi\boldsymbol{]}\varphi\rfloor^\texttt{T}$ is false.
This results from announcing $\psi$, where all worlds get discarded in which $\psi$ does not hold.
Such a world is $\texttt{i}_1$, in which $\psi$ does not hold initially (for our notion of validity).
However, in world $\texttt{i}_2$ the formula $\psi$ holds.
We can see this if we have a look at the first two lines as presented above for $\psi$:
\begin{gather*}
    \texttt{(($\lambda$x.\ \_)($\texttt{i}_1$ := True, $\texttt{i}_2$ := True), $\texttt{i}_1$) := False} \\
    \texttt{(($\lambda$x.\ \_)($\texttt{i}_1$ := True, $\texttt{i}_2$ := True), $\texttt{i}_2$) := True\;\;}
\end{gather*}

Thus, ultimately the term $\lfloor\boldsymbol{[!}\psi\boldsymbol{]}\varphi\rfloor^\texttt{T}$ evaluates in the given context, where $i_1$ has been discarded, as follows:
\[
    \texttt{(($\lambda$x.\ \_)($\texttt{i}_1$ := False, $\texttt{i}_2$ := True), $\texttt{i}_2$) := False}
\]

As a consequence, the validity function is defined in such a way that it checks validity not only for all worlds, but for all domains and worlds.
Otherwise, the observed but undesired countermodel to necessitation may occur.

\vspace{.5em}
\lstinline[language=Isabelle]|abbreviation pvalid :: "$\tau$$\Rightarrow$bool" ("$\lfloor$_$\rfloor$")|\\
\hspace*{3em}\lstinline[language=Isabelle]|where "$\lfloor$_$\rfloor$ $\equiv$ $\forall$W.$\forall$w. W w $\longrightarrow$ $\varphi$ W w"|

\vspace{.5em}
The definitions as introduced in this section are intended to be hidden from the user, who can construct formulas directly in PAL syntax.
The unfolding of these definitions is then handled automatically by Isabelle/HOL.

The SSE approach in \logikey\  also supports the encoding and inspection of concrete models. For example, 
let $W=\{w1, w2, w3\} $ and let $p$ be true at $w1$ and $w2$, but false at $w3$.  Let the relations (given as partitions, i.e. lists of equivalence classes) be $R_a = [[w1,w2],[w3]]$ and $R_b = [[w1],[w2,w3]]$.  Then we have $M,w1 \models p \wedge  K_ap \wedge K_bp \wedge \neg K_aK_bp$. In Isabelle/HOL we can encode this as follows:

\vspace{.5em}
\lstinline[language=Isabelle]|$\textcolor{gray}{\text{(* Concrete models can be defined and studied *)}}$|\\
\hspace*{1.5em}\lstinline[language=Isabelle, basicstyle=\small\ttfamily]|lemma assumes: "W = ($\lambda$x. x = w1 $\vee$ x = w2 $\vee$ x = w3)"|\\
\hspace*{9em}\lstinline[language=Isabelle, basicstyle=\small\ttfamily]|"w1 $\not =$ w2" "w1 $\not =$ w3" "w2 $\not =$ w3"|\\
\hspace*{9em}\lstinline[language=Isabelle, basicstyle=\small\ttfamily]|"p W w1" "p W w2" "$\neg$(p W w3)"|\\
\hspace*{9em}\lstinline[language=Isabelle, basicstyle=\small\ttfamily]|"a w1 w1" "a w1 w2" "a w2 w1"|\\
\hspace*{9em}\lstinline[language=Isabelle, basicstyle=\small\ttfamily]|"a w2 w2" "$\neg$(a w1 w3)" "$\neg$(a w3 w1)"|\\
\hspace*{9em}\lstinline[language=Isabelle, basicstyle=\small\ttfamily]|"$\neg$(a w2 w3)" "$\neg$(a w3 w2)" "a w3 w3"|\\
\hspace*{9em}\lstinline[language=Isabelle, basicstyle=\small\ttfamily]|"b w1 w1" "$\neg$(b w1 w2)" "$\neg$(b w2 w1)"|\\
\hspace*{9em}\lstinline[language=Isabelle, basicstyle=\small\ttfamily]|"b w2 w2" "$\neg$(b w1 w3)" "$\neg$(b w3 w1)"|\\
\hspace*{9em}\lstinline[language=Isabelle, basicstyle=\small\ttfamily]|"b w2 w3" "b w3 w2" "b w3 w3"|\\
\hspace*{6em}\lstinline[language=Isabelle, basicstyle=\small\ttfamily]|shows "((p $\boldsymbol{\wedge}$ ($\textbf{K}_\texttt{a}$ p) $\boldsymbol{\wedge}$ ($\textbf{K}_\texttt{b}$ p)) $\boldsymbol{\wedge}$ $\boldsymbol{\neg}$($\textbf{K}_\texttt{a}$ ($\textbf{K}_\texttt{b}$ p))) W w1"|\\
\hspace*{3em}\lstinline[language=Isabelle, basicstyle=\small\ttfamily]|unfolding Defs|\\
\hspace*{4.5em}\lstinline[language=Isabelle, basicstyle=\small\ttfamily]|nitpick[satisfy, atoms=w1 w2 w3] $\textcolor{gray}{\text{(* model *)}}$|\\
\hspace*{4.5em}\lstinline[language=Isabelle, basicstyle=\small\ttfamily]|using assms(1) assms(5) assms(6) assms(7)|\\
\hspace*{7.5em}\lstinline[language=Isabelle, basicstyle=\small\ttfamily]|assms(9) assms(12) assms(21) assms(23) by blast $\textcolor{gray}{\text{(* proof *)}}$|

\vspace{.5em}
Nitpick generates a model that satisfies these constraints, and the provers in Isabelle/HOL subsequently prove the validity of this claim as expected.
We could of course deepen such experiments here and specify and inspect further models in full detail. This is however left for further work. Further work also includes experimentation with Isabelle/HOL as an educational tool, including e.g.~the exploration and study of PAL models in classroom. Generally, the aspect of (counter-)model finding in the SSE approach deserves more attention in future work.

\section{Soundness of the Embedding}

To show that our embedding is sound, we exploit the following mapping of Kripke frames into Henkin models.

\begin{definition}[\textbf{Henkin Model $\mathcal{H}^\mathcal{M}$ for PAL model $\mathcal{M}$}]
	Let $\mathcal{A}$ be a group of agents.
	For any PAL model $\mathcal{M} = \langle W, \{R_i\}_{i \in \mathcal{A}}, V \rangle$, we define a corresponding Henkin model $\mathcal{H}^\mathcal{M}$.
	Thus, let a PAL model $\mathcal{M} = \langle W, \{R_i\}_{i \in \mathcal{A}}, V \rangle$ be given.
	Moreover, assume that $p^1, \dots, p^m \in \mathcal{P}$, for $m \geq 1$, are the only propositional symbols of PAL.
	Remember that our embedding requires the corresponding signature in HOL to provide constant symbols $p^j_\sigma$ such that $\lfloor p^j \rfloor = p^j_\sigma$ for $j = 1, \dots, m$.
    Moreover, for each $R_{i} \in \mathcal{M}$ we require a corresponding constant symbol $r^{i}_\alpha$ in the signature of HOL (this is similar to what we do for the $p^i \in P$).
	A Henkin model $\mathcal{H}^\mathcal{M} = \langle \{\mathcal{D}_\alpha\}_{\alpha \in T}, I \rangle$ for $\mathcal{M}$ is now defined as follows:
	$\mathcal{D}_i$ is chosen as the set of possible worlds $W$;
	all other sets $\mathcal{D}_{\alpha\rightarrow\beta}$ are chosen as (not necessary full) sets of functions from $\mathcal{D}_\alpha$ to $\mathcal{D}_\beta$.
	For all $\mathcal{D}_{\alpha \rightarrow \beta}$ the rule that every term $t_{\alpha \rightarrow \beta}$ must have a denotation in $\mathcal{D}_{\alpha \rightarrow \beta}$ must be obeyed (Denotatpflicht).
	In particular, it is required that $\mathcal{D}_\sigma$ and $\mathcal{D}_\alpha$ contain the elements $Ip^j_\sigma$ and $Ir^{i}_\alpha$, respectively.
	The interpretation function $I$ of $\mathcal{H}^\mathcal{M}$ is defined as follows:
	
	\begin{enumerate}
		\item For $j = 1,\dots,m$, \ $Ip^j_\sigma \in \mathcal{D}_\sigma$ chosen s.t. $Ip^j_\sigma(d, s) = T$ iff $s \in V(p^j)$ in $\mathcal{M}$.
		\item For $k = 1,\dots, |\mathcal{A}|$, $Ir^{i}_\alpha\in \mathcal{D}_\alpha$ chosen s.t. $Ir^{i}_\alpha(s,u) = T$ iff $u \in R_{i}(s)$ in $\mathcal{M}$.
		\item For the logical connectives $\neg, \vee, \Pi$ and $=$ of HOL the interpretation function $I$ is defined as usual.
	\end{enumerate}
	
	None of these choices is in conflict with the necessary requirements of a Henkin model.
\end{definition}

\begin{lemma} \label{lem:lemma}
	Let $\mathcal{H}^\mathcal{M}$ be a Henkin model for a PAL model $\mathcal{M} = \langle W, \{R_i\}_{i \in \mathcal{A}}, V \rangle$. For all PAL formulas $\delta$, arbitrary variable assignments g, sets of worlds  $d=W$ (the \textit{evaluation domains}) and worlds $s$ it holds: $$\mathcal{M},s \models \delta \text{ if and only if } \llbracket \lfloor \delta \rfloor D_\sigma S_i \rrbracket^{\mathcal{H}^\mathcal{M}, g[d/D_\sigma][s/S_i]} = T$$
\end{lemma}

\begin{proof}
	We start with the case where $\delta$ is $p^j$.
	We have:
	
	\vspace{.5em}
	\noindent
	\hspace*{2.1em}$\llbracket \lfloor p^j \rfloor \ D \ S \rrbracket^{\mathcal{H}^\mathcal{M}, g[d/D_\sigma][s/S_i]} = T$\\
	\hspace*{.6em}$\Leftrightarrow \ \llbracket (^A(p^j_\sigma))_\tau \ D \ S \rrbracket^{\mathcal{H}^\mathcal{M}, g[d/D_\sigma][s/S_i]} = T$\\
	\hspace*{.6em}$\Leftrightarrow \ \llbracket D \ S \wedge p^j_\sigma\ S \rrbracket^{\mathcal{H}^\mathcal{M}, g[d/D_\sigma][s/S_i]} = T$\\
	\hspace*{.6em}$\Leftrightarrow \ s \in d=W \text{ and } Ip^j_\sigma(s) = T \hfill \text{(by definition of $\mathcal{H}^\mathcal{M}$)}$\\
	\hspace*{.6em}$\Leftrightarrow \ M,s \models p^j$
	
	\vspace{1em}
	\noindent
	\textit{Induction hypothesis}: For sentences $\delta$' structurally smaller than $\delta$ we have: For all assignments $g$, domains $d$ and worlds $s, \llbracket \lfloor \delta' \rfloor\ D\ S \rrbracket^{\mathcal{H}^\mathcal{M}, g[d/D_\sigma][s/S_i]} = T$ if and only if $M,s \models \delta'$.
	
	\vspace{.5em}
	\noindent
	We consider each inductive case in turn:
	
	\vspace{1em}
	\noindent
	$\delta =\ \neg \varphi$\\
	\hspace*{.6em}$\Leftrightarrow\ \llbracket \lfloor \neg \varphi \rfloor \ D\ S \rrbracket^{\mathcal{H}^\mathcal{M}, g[d/D_\sigma][s/S_i]} = T$\\
	\hspace*{.6em}$\Leftrightarrow\ \llbracket (\neg_{\tau \rightarrow \tau} \lfloor \varphi \rfloor )\ D \ S \rrbracket^{\mathcal{H}^\mathcal{M}, g[d/D_\sigma][s/S_i]} = T$\\
	\hspace*{.6em}$\Leftrightarrow\ \llbracket \neg (\lfloor \varphi \rfloor\ D \ S) \rrbracket^{\mathcal{H}^\mathcal{M}, g[d/D_\sigma][s/S_i]} = T$\hspace*{5em}$(\text{since } (\neg_{\tau \rightarrow \tau} (\lfloor \varphi \rfloor) \ D \ S) =_{\beta \eta} \neg (\lfloor \varphi \rfloor \ D \ S))$\\
	\hspace*{.6em}$\Leftrightarrow\ \llbracket \lfloor \varphi \rfloor \ D \ S \rrbracket^{\mathcal{H}^\mathcal{M}, g[d/D_\sigma], [s/S_i]} = F$\\
	\hspace*{.6em}$\Leftrightarrow\ M,s \not \models \varphi \hfill (\text{by induction hypothesis})$\\
	\hspace*{.6em}$\Leftrightarrow\ M,s \models \neg \varphi$
	
	\vspace{1em}
	\noindent
	$\delta =\ \varphi \vee \psi$\\
	\hspace*{.6em}$\Leftrightarrow\ \llbracket \lfloor \varphi \vee \psi \rfloor \ D \ S \rrbracket^{\mathcal{H}^\mathcal{M}, g[d/D_\sigma], [s/S_i]} = T$\\
	\hspace*{.6em}$\Leftrightarrow\ \llbracket (\lfloor \varphi \rfloor \vee_{\tau \rightarrow \tau \rightarrow \tau} \lfloor \psi \rfloor) \ D \ S \rrbracket^{\mathcal{H}^\mathcal{M}, g[d/D_\sigma], [s/S_i]} = T$\\
	\hspace*{8.35em}$\hspace{5em}(\text{since } ((\lfloor \varphi \rfloor \vee_{\tau \rightarrow \tau \rightarrow \tau} \lfloor \psi \rfloor) \ D \ S) =_{\beta \eta} (\lfloor \varphi \rfloor \ D \ S) \vee (\lfloor \psi \rfloor) \ D \ S))$\\
	\hspace*{.6em}$\Leftrightarrow\ \llbracket (\lfloor \varphi \rfloor \ D \ S) \vee (\lfloor \psi \rfloor \ D \ S) \rrbracket^{\mathcal{H}^\mathcal{M}, g[d/D_\sigma], [s/S_i]} = T$\\
	\hspace*{.6em}$\Leftrightarrow\ \llbracket \lfloor \varphi \rfloor \ D \ S \rrbracket^{\mathcal{H}^\mathcal{M}, g[d/D_\sigma], g[s/S_i]} = T \text{ or } \ \llbracket \lfloor \psi \rfloor \ D \ S \rrbracket^{\mathcal{H}^\mathcal{M}, g[d/D_\sigma], [s/S_i]} = T$\\
	\hspace*{.6em}$\Leftrightarrow\ M,s \models \varphi \text{ or } M,s \models \psi \hfill (\text{by induction hypothesis})$\\
	\hspace*{.6em}$\Leftrightarrow\ M,s \models \varphi \vee \psi$
	
	\vspace{1em}
	\noindent
	$\delta =\ K\ r^{i}\ \varphi$\\
	\hspace*{.6em}$\Leftrightarrow\ \llbracket \lfloor K\ r^{i}\ \varphi \rfloor \ D \ S \rrbracket^{\mathcal{H}^\mathcal{M}, g[d/D_\sigma][s/S_i]} = T$\\
	\hspace*{.6em}$\Leftrightarrow\ \llbracket K_{\alpha \rightarrow \tau \rightarrow \tau} \ r^{i} \ \lfloor \varphi \rfloor\ D \ S \rrbracket^{\mathcal{H}^\mathcal{M}, g[d/D_\sigma][s/S_i]} = T
	$\\
	\hspace*{.6em}$\Leftrightarrow\ \llbracket \forall Y_i (\neg (D\ Y\ \wedge\ r^{i} \ S \ Y)\ \vee \lfloor \varphi \rfloor \ D \ Y)\rrbracket^{\mathcal{H}^\mathcal{M}, g[d/D_\sigma][s/S_i]} = T$\\
	\hspace*{.6em}$\Leftrightarrow\ \text{For all}\ a \in \mathcal{D}_i \text{ we have } \llbracket (\neg (D\ Y\ \wedge\ r^{i} \ S \ Y)\ \vee \lfloor \varphi \rfloor \ D \ Y) \rrbracket^{\mathcal{H}^\mathcal{M}, g[d/D_\sigma][s/S_i][a/Y_i]} = T$\\
	\hspace*{.6em}$\Leftrightarrow\ \text{For all}\ a \in \mathcal{D}_i \text{ we have } \llbracket \neg (D\ Y\ \wedge\ r^{i} \ S \ Y)\ \rrbracket^{\mathcal{H}^\mathcal{M}, g[d/D_\sigma][s/S_i][a/Y_i]} = T \text{ or}$\\
	\hspace*{2.1em}$\llbracket \lfloor \varphi \rfloor \ D \ Y \rrbracket^{\mathcal{H}^\mathcal{M}, g[d/D_\sigma][s/S_i][a/Y_i]} = T\\$
	\hspace*{.6em}$\Leftrightarrow\ \text{For all}\ a \in \mathcal{D}_i \text{ we have } \llbracket \neg D\ Y\ \vee\ \neg r^{i} \ S \ Y\ \rrbracket^{\mathcal{H}^\mathcal{M}, g[d/D_\sigma][s/S_i][a/Y_i]} = T \text{ or}$\\
	\hspace*{2.1em}$\llbracket \lfloor \varphi \rfloor \ D \ Y \rrbracket^{\mathcal{H}^\mathcal{M}, g[d/D_\sigma][a/Y_i]} = T \hspace*{17em}(S\not\in free(\lfloor \varphi \rfloor))$\\
	\hspace*{.6em}$\Leftrightarrow\ \text{For all}\ a \in \mathcal{D}_i \text{ we have } \llbracket \neg D\ Y\rrbracket^{\mathcal{H}^\mathcal{M}, g[d/D_\sigma][s/S_i][a/Y_i]} = T \text{ or}$\\
	\hspace*{2.1em}$\llbracket \neg r^{i} \ S \ Y\ \rrbracket^{\mathcal{H}^\mathcal{M}, g[d/D_\sigma][s/S_i][a/Y_i]} = T \text{ or } \llbracket \lfloor \varphi \rfloor \ D \ Y \rrbracket^{\mathcal{H}^\mathcal{M}, g[d/D_\sigma][a/Y_i]} = T$\\
	\hspace*{.6em}$\Leftrightarrow\ \text{For all}\ a \in \mathcal{D}_i \text{ we have } \llbracket D\ Y\rrbracket^{\mathcal{H}^\mathcal{M}, g[d/D_\sigma][s/S_i][a/Y_i]} = F \text{ or}$\\
	\hspace*{2.1em}$\llbracket r^{i} \ S \ Y\ \rrbracket^{\mathcal{H}^\mathcal{M}, g[d/D_\sigma][s/S_i][a/Y_i]} = F \text{ or } \llbracket \lfloor \varphi \rfloor \ D \ Y \rrbracket^{\mathcal{H}^\mathcal{M}, g[d/D_\sigma][a/Y_i]} = T$
	\hfill (by ind.~hyp.)\\
	\hspace*{.6em}$\Leftrightarrow\ \text{For all}\ a \in \mathcal{D}_i \text{ we have } (a \not \in d \text{ or } a \not \in r^{i}(s)) \text{ or } M,a \models \varphi$\\
	\hspace*{.6em}$\Leftrightarrow\ \text{For all}\ a \in \mathcal{D}_i \text{ we have } a\not \in (d \cap r^{i}(s)) \text { or } M,a \models \varphi$\\
	\hspace*{.6em}$\Leftrightarrow\ M,s \models K\ r^{i} \varphi$
	
	\vspace{1em}
	\noindent
	$\delta =\ [!\varphi]\psi$\\
	\hspace*{.6em}$\Leftrightarrow\ \llbracket \lfloor [! \varphi ] \psi \rfloor\ D\ S\rrbracket^{\mathcal{H}^\mathcal{M},g[d/D_\sigma][s/S_i]}=T$\\
	\hspace*{.6em}$\Leftrightarrow\ \llbracket ([! \ \cdot \ ] \cdot_{\tau \rightarrow \tau \rightarrow \tau} \lfloor \varphi \rfloor \lfloor \psi \rfloor)\ D\ S\rrbracket^{\mathcal{H}^\mathcal{M},g[d/D_\sigma][s/S_i]}=T$\\
	\hspace*{.6em}$\Leftrightarrow\ \llbracket (\neg \lfloor \varphi \rfloor \ D \ S) \vee (\lfloor \psi \rfloor\ (\lambda Y_i \ D\ Y\ \wedge \lfloor \varphi \rfloor\ D\ Y)\ S)\rrbracket^{\mathcal{H}^\mathcal{M},g[d/D_\sigma][s/S_i]}=T$\\
	\hspace*{.6em}$\Leftrightarrow\ \llbracket (\neg \lfloor\varphi \rfloor \ D \ S)\rrbracket^{\mathcal{H}^\mathcal{M},g[d/D_\sigma][s/S_i]}=T\, \text{or}\, \llbracket(\lfloor \psi \rfloor\ (\lambda Y_i \ D\ Y\ \wedge \lfloor \varphi \rfloor\ D\ Y)\ S)\rrbracket^{\mathcal{H}^\mathcal{M},g[d/D_\sigma][s/S_i]}=T$\\
	\hspace*{.6em}$\Leftrightarrow\ \llbracket (\lfloor \varphi \rfloor \ D \ S)\rrbracket^{\mathcal{H}^\mathcal{M},g[d/D_\sigma][s/S_i]}=F \text{ or } \llbracket(\lfloor \psi \rfloor\ (\lambda Y_i \ D\ Y\ \wedge \lfloor \varphi \rfloor\ D\ Y)\ S)\rrbracket^{\mathcal{H}^\mathcal{M},g[d/D_\sigma][s/S_i]}=T$\\
	\hspace*{.6em}$\Leftrightarrow\ M,s \not \models \varphi \text{ or } \llbracket(\lfloor \psi \rfloor\ (\lambda Y_i \ D\ Y\ \wedge \lfloor \varphi \rfloor\ D\ Y)\ S)\rrbracket^{\mathcal{H}^\mathcal{M},g[d/D_\sigma][s/S_i]}=T$\hfill (by ind.~hyp.) \\
	\hspace*{.6em}$\Leftrightarrow\ M,s \not \models \varphi \text{ or } M^{!\varphi},s \models \psi$\hfill(\textbf{Justification})\\
	\hspace*{.6em}${\Leftrightarrow\ M , s \models [! \varphi ] \psi} $
	\begin{mdframed}
	    \textbf{Justification:}\\
		From the induction hypothesis follows that $\llbracket \lfloor \psi \rfloor\ D\ S \rrbracket^{\mathcal{H}^\mathcal{M}, g[d/D_\sigma][s/S_i]} = T$ if and only if $\mathcal{M},s \models \psi$. In order to see how $\llbracket \lfloor \psi \rfloor\ (\lambda Y_i\ D\ Y\ \wedge \lfloor \varphi \rfloor\ D\ Y)\ S \rrbracket^{\mathcal{H}^\mathcal{M}, g[d/D_\sigma][s/S_i]} = T$ and $M^{!\varphi}, s \models \psi$ relate, we remind ourselves of the definition of the model $\mathcal{M}$ after $\varphi$ is publicly announced:	
		$\mathcal{M}^{!\varphi} = \langle W^{!\varphi}, \{R_i^{!\varphi}\}_{i\in\mathcal{A}}, V^{!\varphi} \rangle$ where $W^{!\varphi} = \{ w \in W \ | \ \mathcal{M},w \models \varphi \}$, 
		$R_i^{!\varphi} = R_i \cap (W^{!\varphi} \times W^{!\varphi})$ for all $i \in \mathcal{A}$, and $V^{!\varphi}(p) = V(p) \cap W^{!\varphi}$ for all $p \in \mathcal{P}$.
		For the embedding this means that $W^{!\varphi}$ retains only worlds in which $\varphi$ is true, while the arrows/relations between worlds remain the same and get evaluated in the embedding as explained.
		By encoding the updated domain as ($\lambda Y_i\ D\ Y\ \wedge \lfloor \varphi \rfloor\ D\ Y$), denoting $\{ w \in W \ | \ \mathcal{M},w \models \varphi \}$ in the given context, we ultimately restrict ourselves to worlds in $W^{!\varphi}$.
	    Relations indirectly get restricted to this new domain ($R_i^{!\varphi} = R_i \cap (W^{!\varphi} \times W^{!\varphi})$), due to the recursively conducted domain checks (see e.g.~the definitions of the public announcement operator), and an analogous argument applies for the evaluation of atomic propositions ($V^{!\varphi}(p) = V(p) \cap W^{!\varphi}$). We thus get: $\llbracket (\lfloor \psi \rfloor \ (\lambda Y_i\ D\ Y\ \wedge \ \lfloor \varphi \rfloor \ D \ Y) \ S) \rrbracket^{\mathcal{H}^\mathcal{M}, g[d/Y_i], [s/S_i]} = T$ if and only if $M^{!\varphi}, s$. Our argument here is informal in order to avoid further technicalities. Formally, another (analogous) inductive argument is required (now on the structure of $\psi$).
	\end{mdframed}
	
	\vspace{1em}
	\noindent
	$\delta\ =\ \mathcal{C}_{G}(\varphi | \psi)$
	
	\noindent
	This case is similar to $K \ r^{i} \ \varphi$.
	The only difference is the construction of the accessibility relation, which now depends on $\varphi$ (and $D$).
	When comparing the definition of relativized common knowledge\footnote{$\mathcal{M},w \models \mathcal{C}^{!\varphi} \psi \text{ iff } (w,v) \in (R_{E_G} \cap (W \times \llbracket \varphi \rrbracket^\mathcal{M}))^+ \text{ implies } \mathcal{M}, v \models \psi$.} with the proposed embedding of $K \ r^{i} \ \varphi$, the analogy becomes apparent;
	the proof is technical and therefore omitted.
\end{proof}

We can now prove the soundness of the embedding.

\begin{theorem}[\textbf{Soundness of the Embedding}]
	\[ \text{If } \models^{\texttt{HOL}} \texttt{vld}(\lfloor \varphi \rfloor) \text{ then } \models^{\texttt{PAL}} \varphi \]
\end{theorem}
\begin{proof}
	The proof is by contraposition.
	Assume $\not \models^{\texttt{PAL}} \varphi$, i.e., there is a PAL model $\mathcal{M} = \langle W, \{R_i\}_{i \in \mathcal{A}}, V\rangle$ and a world $s \in W$, such that $\mathcal{M},s \not \models \varphi$.
	By Lemma \ref{lem:lemma} it holds that $\llbracket \lfloor \varphi \rfloor D_\sigma S_i \rrbracket^{\mathcal{H}^\mathcal{M}, g[d/D_\sigma][s/S_i]} = F$ (for some $g$ and $d = W$) in Henkin model $\mathcal{H}^\mathcal{M} = \langle \{\mathcal{D}_\alpha\}_{\alpha \in T}, I\rangle$ for $M$.
	Now, $\llbracket \lfloor \varphi \rfloor D_\sigma S_i \rrbracket^{\mathcal{H}^\mathcal{M}, g[d/D_\sigma][s/S_i]} = F$  implies that $\llbracket \forall D_\sigma \forall S_i (D\ S\ \longrightarrow\ \lfloor \varphi \rfloor \ D\ S)\rrbracket^{\mathcal{H}^\mathcal{M}, g} = \llbracket \texttt{vld} \lfloor \varphi \rfloor \rrbracket^{\mathcal{H}^\mathcal{M}, g} = F$. Hence, $\mathcal{H}^\mathcal{M} \not \models^\texttt{HOL} \texttt{vld} \lfloor \varphi \rfloor$.
\end{proof}	

The completeness of our embedding of PAL in HOL is addressed in the next chapter. For this, we show that standard axioms and inference rules of PAL can be inferred from our embedding. Except for two axioms (which seem to require induction) all these meta-theoretical proofs were found fully automatically.

\section{Experiments (including Completeness Aspects)}
\subsection{Proving Axioms and Rules of Inference of PAL in HOL}

The presented SSE of PAL is able to prove the following axioms and rules of inference as presented for PAL by Baltag and Renne \cite[Supplement F]{sep-dynamic-epistemic}:

\vspace{1em}
\begin{tabular}{p{10em}l}
    \multicolumn{2}{l}{\textbf{System K}}\\
    -- & All substitutions instances of propositional tautologies\\
    Axiom K & $K_i(\varphi \rightarrow \psi) \rightarrow (K_i \varphi \rightarrow K_i \psi)$\\
    Modus ponens & From $\varphi$ and $\varphi \rightarrow \psi$ infer $\psi$\\
Necessitation & From $\varphi$ infer $K_i \varphi$\\[1em]
\end{tabular}

\begin{tabular}{p{10em}l}
    \multicolumn{2}{l}{\textbf{System $\mathcal{S}5$}}\\
    Axiom T & $K_i \varphi \rightarrow \varphi$\\
    Axiom 4 & $K_i \varphi \rightarrow K_i K_i \varphi$\\
    Axiom 5 & $\neg K_i\noindent \varphi \rightarrow K_i \neg K_i \varphi$\\[1em]
\end{tabular}

\begin{tabular}{p{10em}l}
    \multicolumn{2}{l}{\textbf{Reduction Axioms}}\\
    Atomic Permanence & $[!\varphi]p \leftrightarrow (\varphi \rightarrow p)$\\
    Conjunction & $[!\varphi](\psi \wedge \chi) \leftrightarrow ([!\varphi]\psi \wedge [!\varphi]\chi)$\\
    Partial Functionality & $[!\varphi]\neg \psi \leftrightarrow (\varphi \rightarrow \neg [!\varphi]\psi)$\\
    Action-Knowledge & $[!\varphi]K_i\psi \leftrightarrow (\varphi \rightarrow K_i (\varphi \rightarrow K_i (\varphi \rightarrow [!\varphi]\psi)))$\\
    -- & $[!\varphi]\mathcal{C}(\chi | \psi) \leftrightarrow (\varphi \rightarrow \mathcal{C}(\varphi \wedge [!\varphi]\chi | [!\varphi]\psi))$\\[1em]
\end{tabular}

\begin{tabular}{p{10em}l}
    \multicolumn{2}{l}{\textbf{Axiom schemes for Relativized Common Knowledge (RCK)}}\\
    $\mathcal{C}$-normality & $\mathcal{C}(\chi | (\varphi \rightarrow \psi)) \rightarrow (\mathcal{C}(\chi | \varphi) \rightarrow \mathcal{C}(\chi | \psi))$\\
    Mix axiom & $\mathcal{C}(\psi | \varphi) \leftrightarrow E(\psi \rightarrow (\varphi \wedge \mathcal{C}(\psi | \varphi)))$\\
    Induction axiom & $(E(\psi \rightarrow \varphi) \wedge \mathcal{C}(\psi | \varphi \rightarrow E(\psi \rightarrow \varphi))) \rightarrow \mathcal{C}(\psi | \varphi)$\\[1em]
\end{tabular}

\begin{tabular}{p{10em}l}
    \multicolumn{2}{l}{\textbf{Rules of Inference}}\\
    Announcement Nec. & from $\varphi$, infer $[!\psi]\varphi$\\
    RCK Necessitation & from $\varphi$, infer $\mathcal{C}(\psi | \varphi)$
\end{tabular}
\vspace{1em}

Automatic proofs in Isabelle/HOL can be found for all axioms except for the right-to-left direction of the mix axiom  and the induction axiom schemata for relativized common knowledge (cf.~Fig.~\ref{fig:bild2} in App.~A). 
While for these two open cases full proof automation (with the current tools) fails, some simple edge cases can nevertheless be proved automatically.\footnote{Should induction proofs be needed to prove the general cases, this will lead to interesting further work: How to best handle structural induction over the shallowly embedded PAL formulas, while still avoiding a deep embedding of PAL in HOL?}
However, to ensure completeness of our SSE of PAL in HOL, we can simply postulate the two axiom schemes for induction and mix and postpone proving that they are in fact already entailed. This, in fact,  illustrates another interesting feature with respect to the rapid prototyping of new logical formalisms in the \logikey\ approach.

The consistency of our embedding, resp.~axiomatization, of PAL in Isabelle/HOL (and also of the additional axioms for induction and mix, and for the wise men puzzle) is confirmed by the model finder Nitpick.

\subsection{Exploring Failures of Uniform Substitution} \label{sec:failures}

The following principles are examples of sentences that are \textit{valid} for atomic propositions $p$, but not \textit{schematically valid} for arbitrary formulas $\varphi$ \cite{holliday2013}. 
The results of our experiments are as expected; proofs can be found for the atomic cases $p$. For the schematic formulas $\varphi$, however, countermodels are reported by the model finder Nitpick (cf.~Fig.~\ref{fig:bild2} in App.~A)
\begin{enumerate}
    \setlength\itemsep{.5em}
	\item {$p \rightarrow \neg [!p](\neg p)$}
 	\item $p \rightarrow \neg [!p](\neg K_i p)$
	\item $p \rightarrow \neg [!p](p \wedge \neg K_i p)$
	\item $(p \wedge \neg K_i p) \rightarrow \neg [!p \wedge \neg K_i p] (p \wedge \neg K_i p)$
	\item $K_i p \rightarrow \neg [!p](\neg K_i p)$
	\item $K_i p \rightarrow \neg [!p](p \wedge \neg K_i p)$
\end{enumerate}

As an example, consider the schematic counterpart of (1) for which Nitpick reports the following countermodel:

\vspace{.5em}
\lstinline[language=Isabelle]|lemma "$\boldsymbol{\lfloor}\varphi \boldsymbol{\rightarrow} \boldsymbol{\neg [!}\varphi\boldsymbol{]}(\boldsymbol{\neg}\varphi) \boldsymbol{\rfloor}$" nitpick oops|\\
\noindent\rule{\textwidth}{0.4pt}

\begin{lstlisting}[basicstyle=\small\ttfamily]
Nitpick found a counterexample for card i = 2:

Free variables:
$\varphi$ = ($\lambda$x. _)
      ((($\lambda$x. _)($\texttt{i}_1$ := True, $\texttt{i}_2$ := True), $\texttt{i}_1$) := True,
       (($\lambda$x. _)($\texttt{i}_1$ := True, $\texttt{i}_2$ := True), $\texttt{i}_2$) := False,
       (($\lambda$x. _)($\texttt{i}_1$ := True, $\texttt{i}_2$ := False), $\texttt{i}_1$) := False,
       (($\lambda$x. _)($\texttt{i}_1$ := True, $\texttt{i}_2$ := False), $\texttt{i}_2$) := False,
       (($\lambda$x. _)($\texttt{i}_1$ := False, $\texttt{i}_2$ := True), $\texttt{i}_1$) := False,
       (($\lambda$x. _)($\texttt{i}_1$ := False, $\texttt{i}_2$ := True), $\texttt{i}_2$) := False,
       (($\lambda$x. _)($\texttt{i}_1$ := False, $\texttt{i}_2$ := False), $\texttt{i}_1$) := False,
       (($\lambda$x. _)($\texttt{i}_1$ := False, $\texttt{i}_2$ := False), $\texttt{i}_2$) := False)
Skolem constant:
  W = ($\lambda$x. _)($\texttt{i}_1$ := True, $\texttt{i}_2$ := True)
  w = $\texttt{i}_1$
\end{lstlisting}

The explanation for this model is similar to the one presented in \S\ref{sec:encoding}.
This output is expected, given that (1) has been shown not to be \emph{schematically} valid \cite{van2012everything}.

\subsection{Example Application: The Wise Men Puzzle}

The wise men puzzle is an interesting riddle in epistemic reasoning.
It is well suited to demonstrate epistemic actions in a multi-agent scenario.
Baldoni~\cite{baldoni1998normal} gave a formalization for this, which later got embedded into Isabelle/HOL by Benzmüller~\cite{J41,J44}.
In the following implementation, these results will be used as a stepping stone. Note that below we are not going to define a specific model for the wise men,  but instead we analyze the puzzle using the semantic consequence relation.

First, the riddle is recited, and then we go into detail on how the uncertainties change.\footnote{A very similar riddle, that is often presented in the literature is the \emph{Muddy Children} puzzle \cite{DBLP:journals/apal/FaginHMV99}.
A difference between these two riddles is that in the version presented here, the agents get asked sequentially, not synchronously.}

\begin{quote}\it \small
	Once upon a time, a king wanted to find the wisest out of his three wisest men.
	He arranged them in a circle so that they can see and hear each other and told them that he would put a white or a black spot on their foreheads and that one of the three spots would certainly be white.
	The three wise men could see and hear each other but, of course, they could not see their faces reflected anywhere.
	The king, then, asked each of them [sequentially] to find out the color of his own spot.
	After a while, the wisest correctly answered that his spot was white.
\end{quote}

\noindent
The already existing encoding by Benzm\"uller puts a particular emphasis on the adequate modeling of common knowledge.
In \cite{C90}, this solution was enhanced by the public announcement operator.
Consequently, common knowledge was no longer statically stated after each iteration, but a dynamic approach was used for this.
Here, the modeling of this riddle has been further improved (e.g., by better parameterizing our notions over groups of agents) and we are automating the puzzle now for four agents instead of three.

Before we can evaluate the knowledge of the first wise man, we need to formulate the initial circumstances and background knowledge.
Let \texttt{a}, \texttt{b}, \texttt{c} and \texttt{d} be the wise men (they are being encoded as relations of type $\alpha$).
It is common knowledge, that each wise man can see the foreheads of the other wise men.
The only doubt a wise man has, is whether he has a white spot on his own forehead or not.
Additionally, it is common knowledge that at least one of the four wise men has a white spot on his forehead.
The rules of the riddle are encoded as follows:

\begin{lstlisting}[frame=None, basicstyle=\small\ttfamily]
$\textcolor{gray}{\text{(* Agents modeled as accessibility relations *)}}$
consts a::"$\alpha$" b::"$\alpha$" c::"$\alpha$" d::"$\alpha$"
abbreviation Agent::"$\sigma$$\Rightarrow$bool" ("$\mathcal{A}$")
  where "$\mathcal{A}$ x $\equiv$ x = a $\boldsymbol{\vee}$ x = b $\boldsymbol{\vee}$ x = c $\boldsymbol{\vee}$ x = d"
axiomatization where group_S5: "S5Agents $\mathcal{A}$"

$\textcolor{gray}{\text{(* Common knowledge: at least one of a, b, c and d has a white spot *)}}$
consts ws::"$\alpha\Rightarrow\sigma$"
axiomatization where WM1: "$\lfloor \textbf{C}_\mathcal{A} \ (^\text{A}\text{ws a} \ \boldsymbol{\vee} \ ^\text{A}\text{ws b} \ \boldsymbol{\vee} \ ^\text{A}\text{ws c} \ \boldsymbol{\vee} \ ^\text{A}\text{ws d}) \rfloor$"
axiomatization where
  $\textcolor{gray}{\text{(* Common knowledge: if x has not a white spot then y know this *)}}$
  WM2ab: "$\lfloor \textbf{C}_\mathcal{A} \ (\boldsymbol{\neg} (^\text{A}\text{ws a}) \boldsymbol{\rightarrow} \textbf{K}_\text{b} (\boldsymbol{\neg} (^\text{A}\text{ws a}) )) \rfloor$" and
  WM2ac: "$\lfloor \textbf{C}_\mathcal{A} \ (\boldsymbol{\neg} (^\text{A}\text{ws a}) \boldsymbol{\rightarrow} \textbf{K}_\text{c} (\boldsymbol{\neg} (^\text{A}\text{ws a}) )) \rfloor$" and
  WM2ad: "$\lfloor \textbf{C}_\mathcal{A} \ (\boldsymbol{\neg} (^\text{A}\text{ws a}) \boldsymbol{\rightarrow} \textbf{K}_\text{d} (\boldsymbol{\neg} (^\text{A}\text{ws a}) )) \rfloor$" and
  WM2ba: "$\lfloor \textbf{C}_\mathcal{A} \ (\boldsymbol{\neg} (^\text{A}\text{ws b}) \boldsymbol{\rightarrow} \textbf{K}_\text{a} (\boldsymbol{\neg} (^\text{A}\text{ws b}) )) \rfloor$" and
  WM2bc: "$\lfloor \textbf{C}_\mathcal{A} \ (\boldsymbol{\neg} (^\text{A}\text{ws b}) \boldsymbol{\rightarrow} \textbf{K}_\text{c} (\boldsymbol{\neg} (^\text{A}\text{ws b}) )) \rfloor$" and
  WM2bd: "$\lfloor \textbf{C}_\mathcal{A} \ (\boldsymbol{\neg} (^\text{A}\text{ws b}) \boldsymbol{\rightarrow} \textbf{K}_\text{d} (\boldsymbol{\neg} (^\text{A}\text{ws b}) )) \rfloor$" and
  WM2ca: "$\lfloor \textbf{C}_\mathcal{A} \ (\boldsymbol{\neg} (^\text{A}\text{ws c}) \boldsymbol{\rightarrow} \textbf{K}_\text{a} (\boldsymbol{\neg} (^\text{A}\text{ws c}) )) \rfloor$" and
  WM2cb: "$\lfloor \textbf{C}_\mathcal{A} \ (\boldsymbol{\neg} (^\text{A}\text{ws c}) \boldsymbol{\rightarrow} \textbf{K}_\text{b} (\boldsymbol{\neg} (^\text{A}\text{ws c}) )) \rfloor$" and
  WM2cd: "$\lfloor \textbf{C}_\mathcal{A} \ (\boldsymbol{\neg} (^\text{A}\text{ws c}) \boldsymbol{\rightarrow} \textbf{K}_\text{d} (\boldsymbol{\neg} (^\text{A}\text{ws c}) )) \rfloor$" and
  WM2da: "$\lfloor \textbf{C}_\mathcal{A} \ (\boldsymbol{\neg} (^\text{A}\text{ws d}) \boldsymbol{\rightarrow} \textbf{K}_\text{a} (\boldsymbol{\neg} (^\text{A}\text{ws d}) )) \rfloor$" and
  WM2db: "$\lfloor \textbf{C}_\mathcal{A} \ (\boldsymbol{\neg} (^\text{A}\text{ws d}) \boldsymbol{\rightarrow} \textbf{K}_\text{b} (\boldsymbol{\neg} (^\text{A}\text{ws d}) )) \rfloor$" and
  WM2dc: "$\lfloor \textbf{C}_\mathcal{A} \ (\boldsymbol{\neg} (^\text{A}\text{ws d}) \boldsymbol{\rightarrow} \textbf{K}_\text{c} (\boldsymbol{\neg} (^\text{A}\text{ws d}) )) \rfloor$"
\end{lstlisting}

The positive counterparts $\texttt{\textbf{C}}_\mathcal{A} \ ((^\text{A}\texttt{ws x}) \boldsymbol{\rightarrow} \texttt{\textbf{K}}_\texttt{y} (^\text{A}\texttt{ws x}))$ for \texttt{x}, \texttt{y} $\in \mathcal{A}$  of the above negative axioms are implied; this is quickly  confirmed by the automated proof tools in Isabelle/HOL. For example, we have (where $\texttt{\textbf{group\_S5}}$
is referring to the S5 properties of the epistemic operators $\texttt{\textbf{K}}_\texttt{y}$):

\begin{lstlisting}[frame=None, basicstyle=\small\ttfamily]
lemma WM2ab$\text{'}$: "$\lfloor \textbf{C}_\mathcal{A} \ ((^\text{A}\text{ws a}) \boldsymbol{\rightarrow} \textbf{K}_\text{b} (^\text{A}\text{ws a})) \rfloor$"
  using WM2ab group_S5 unfolding Defs by (smt (z3))
\end{lstlisting}

\noindent
Now the king asks whether the first wise man, say $a$, knows if he has a white spot or not.
Assume that $a$ publicly answers that he does not. This is a public announcement of the form: $\neg (\texttt{K}_\texttt{a} (^\texttt{A}\texttt{ws a}) \vee (\texttt{K}_\texttt{a} \neg (^\texttt{A}\texttt{ws a})))$.
Again, a wise man gets asked by the king whether he knows if he has a white spot or not.
Now it's $b$'s turn, and assume that $b$ also announces that he does not know whether he has a white spot on his forehead.\footnote{The case where neither $a$ nor $b$ can correctly infer the color of their forehead when being asked by the king is the most challenging case; we only discuss this one here.} 
The third wise man, $c$, is also unable to tell whether he has a white spot or not.

When asked, $d$ is able to give the right answer, namely that he has a white spot on his forehead.
We can prove this automatically in Isabelle/HOL:\footnote{The experiments have been carried out using Isabelle 2021 on a Lenovo ThinkPad T480s with Intel\textregistered Core i7-8550U QuadCore@1.8Ghz and 16GB RAM.}

\begin{lstlisting}[frame=None, basicstyle=\small\ttfamily]
theorem whitespot_d: "$\lfloor \boldsymbol{[!}\boldsymbol{\neg} \textbf{K}_\text{a} (^\text{A}\text{ws a})\boldsymbol{]}(\boldsymbol{[!}\boldsymbol{\neg} \textbf{K}_\text{b} (^\text{A}\text{ws b})\boldsymbol{]}(\boldsymbol{[!}\boldsymbol{\neg} \textbf{K}_\text{c} (^\text{A}\text{ws c})\boldsymbol{]}(\textbf{K}_\text{d} (^\text{A}\text{ws d}))))\rfloor$"
  using WM1 WM2ba WM2ca WM2cb WM2da WM2db WM2dc
  unfolding Defs by (smt (verit))
\end{lstlisting}

Alternatively we e.g.~get:

\begin{lstlisting}[frame=None, basicstyle=\small\ttfamily]
theorem whitespot_d$\text{'}$:
  "$\lfloor \boldsymbol{[!}\boldsymbol{\neg} ((\textbf{K}_\text{a} (^\text{A}\text{ws a})) \boldsymbol{\vee} (\textbf{K}_\text{a} (\boldsymbol{\neg} ^\text{A}\text{ws a})))\boldsymbol{]}(\boldsymbol{[!}\boldsymbol{\neg} ((\textbf{K}_\text{b} (^\text{A}\text{ws b})) \boldsymbol{\vee} (\textbf{K}_\text{b} (\boldsymbol{\neg} ^\text{A}\text{ws b})))\boldsymbol{]}($
     $\boldsymbol{[!}\boldsymbol{\neg} ((\textbf{K}_\text{c} (^\text{A}\text{ws c})) \boldsymbol{\vee} (\textbf{K}_\text{c} (\boldsymbol{\neg} ^\text{A}\text{ws c})))\boldsymbol{]}(\textbf{K}_\text{d} (^\text{A}\text{ws d}))))\rfloor$"
  using whitespot_c
  unfolding Defs sledgehammer[verbose]() $\textcolor{gray}{\text{(* finds proof *)}}$
  $\textcolor{gray}{\text{(* reconstruction timeout *)}}$
\end{lstlisting}

\section{Comparison with Related Work}
In related work \cite{BenthemEGS18}, van Benthem, van Eijck and colleagues have studied a  \textit{``faithful representation of DEL [dynamic epistemic logic] models as so-called knowledge structures that allow for symbolic model checking''.}
The authors show that such an approach enables efficient and effective reasoning in epistemic scenarios with state-of-the-art Binary Decision Diagram (BDD) reasoning technology, outperforming other existing methods \cite{DEMO,DEMO-S5} to automate DEL reasoning.
Further related work \cite{van2006model} demonstrates how dynamic epistemic terms can be formalized in temporal epistemic terms to apply the model checkers MCK \cite{gammie2004mck} or MCMAS \cite{raimondi2004verification}.
Our approach differs in various respects, including: 

\begin{description}
    \item[External vs.~internal representation transformation:]
    Instead of writing external (e.g.~Haskell-)code to realize the required conversions from DEL into Boolean representations, we work with logic-internal conversions into HOL, provided in the form of a set of equations stated in HOL itself (thereby heavily exploiting the virtues of $\lambda$-conversion).
    Our encoding is concise (only about 50 lines in Isabelle/HOL) and human readable.

    \item[Meta-logical reasoning:]
    \sloppy
    Since our conversion ``code'' is provided within the \mbox{(meta-)}logic environment itself, the conversion becomes better controllable and even amenable to formal verification.
    Moreover, as we have also demonstrated in this article, meta-logical studies about the embedded logics and their embedding in HOL are well-supported in our approach.

    \item[Scalability beyond propositional reasoning:]
    Real-world applications often require differentiation between entities/individuals, their properties and functions defined on them.
    Moreover, quantification over entities (or properties and functions) supports generic statements that are not supported in propositional DEL.
    In contrast to the related work, the shallow semantical embedding approach very naturally scales for first-order and  higher-order extensions of the embedded logics; for more details on this we refer to \cite{J41,J44} and the references therein.

    \item[Reuse of automated theorem proving and model finding technology:]
    Both the related work and our approach reuse state-of-the-art automated reasoning technology.
    In our case, this includes world-leading first-order and higher-order theorem provers and model finders already integrated with Isabelle/HOL \cite{sledgehammer}.
    These tools in turn internally collaborate with the latest SMT and SAT solving technology.
    The burden to organize and orchestrate the technical communication with and between these tools is taken away from us by reuse of respective solutions as already provided in Isabelle/HOL (and recursively also within the integrated theorem provers).
    Well established and robustly supported language formats (e.g.~TPTP syntax, \url{http://www.tptp.org}) are reused in these nested transformations. These cascades of already supported logic transformations are one reason why our embedding approach readily scales for automating reasoning beyond just propositional DEL. 
\end{description}

We are convinced, for reasons as discussed above, that our approach is particularly well suited for the exploration and rapid prototyping of new logics (and logic combinations) and their embeddings in HOL, and for the study of their meta-logical properties, in particular, when it comes to first-order and higher-order extensions of DEL. 
At the same time, we share with the related work by van Benthem, van Eijck and colleagues  a deep interest in practical (object-level) applications, and therefore practical reasoning performance is obviously also of high relevance.
In this regard, however, we naturally assume a performance loss in comparison to hand-crafted, specialist solutions.
Previous studies in the context of first-order modal logic theorem proving nevertheless have shown that this is not always the case \cite{C62}.

\section{Conclusion}
A shallow semantical embedding of public announcement logic with relativized common knowledge in classical higher-order logic has been presented. Our implementation of this embedding in Isabelle/HOL delivers promising initial results, as evidenced by the effective automation of the prominent wise men puzzle. 
In particular, we have shown how model-changing behavior can be adequately and elegantly addressed in our embedding approach.
With reference to uniform substitution, we saw that our embedding enables the study of meta-logical properties of public  announcement logic, and object-level reasoning has been demonstrated by a first-time automation of the wise men puzzle encoded in public announcement logic with a relativized common knowledge operator.

\section*{Acknowledgments}
We thank the anonymous reviewers of this article for their valuable feedback and comments that helped us improve this article. We also thank the reviewers of our related earlier paper presented at the 3rd DaL\'{i} Workshop on Dynamic Logic.

\bibliographystyle{abbrvdin}
\bibliography{bibliography}

\pagebreak
\begin{appendix}
\section{Isabelle/HOL sources}
The sources of our modeling and experiments in Isabelle/HOL are presented in Figs.~\ref{fig:bild1}-\ref{fig:bild4}.

\begin{figure}[!htb]
    \centering
    \colorbox{gray!30}{\includegraphics[width=.93\textwidth]{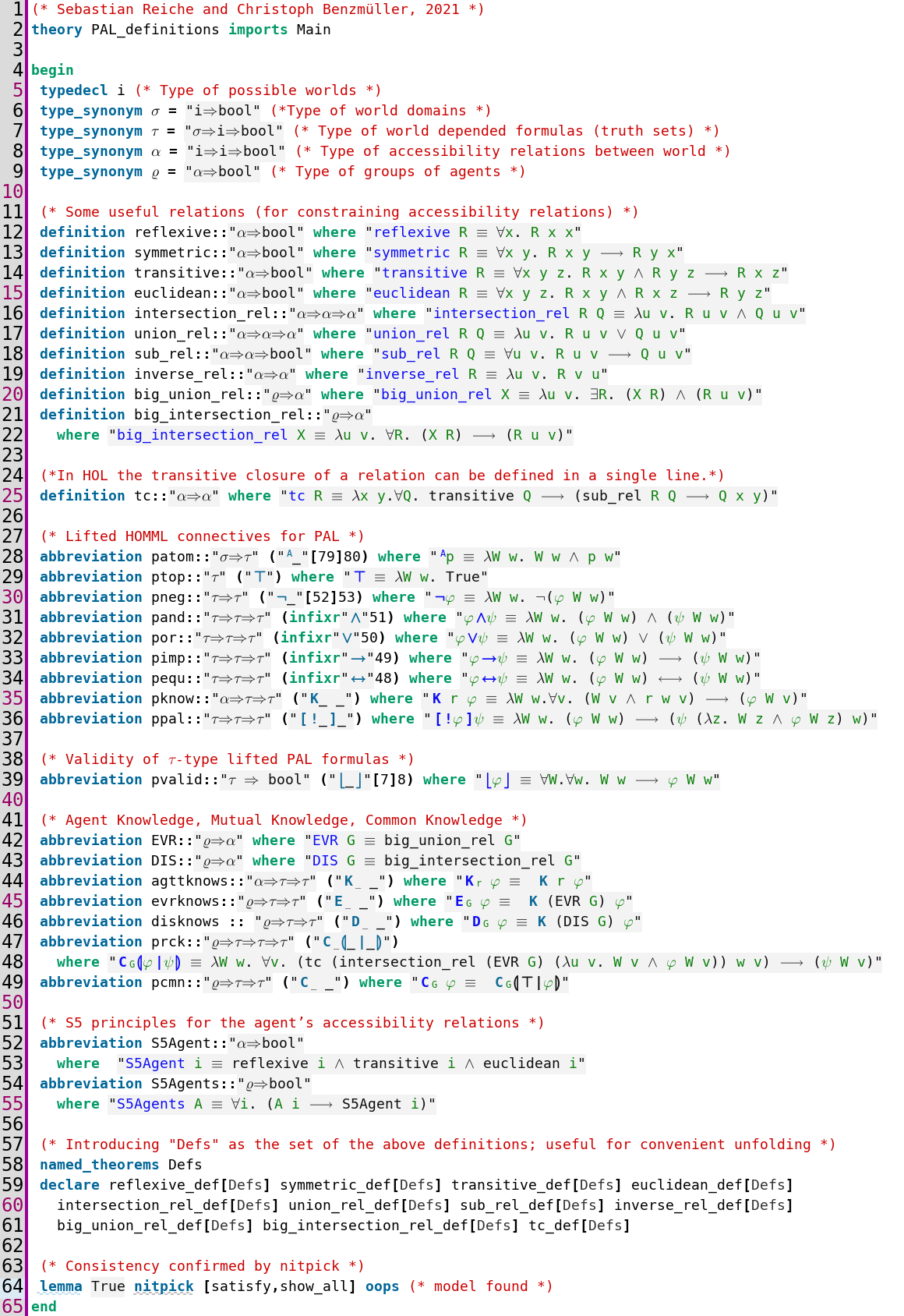}}
    \caption{Embedding of PAL in HOL}
\label{fig:bild1}
\end{figure}
  
\begin{figure}[!htb]
    \centering
    \colorbox{gray!30}{\includegraphics[width=.95\textwidth]{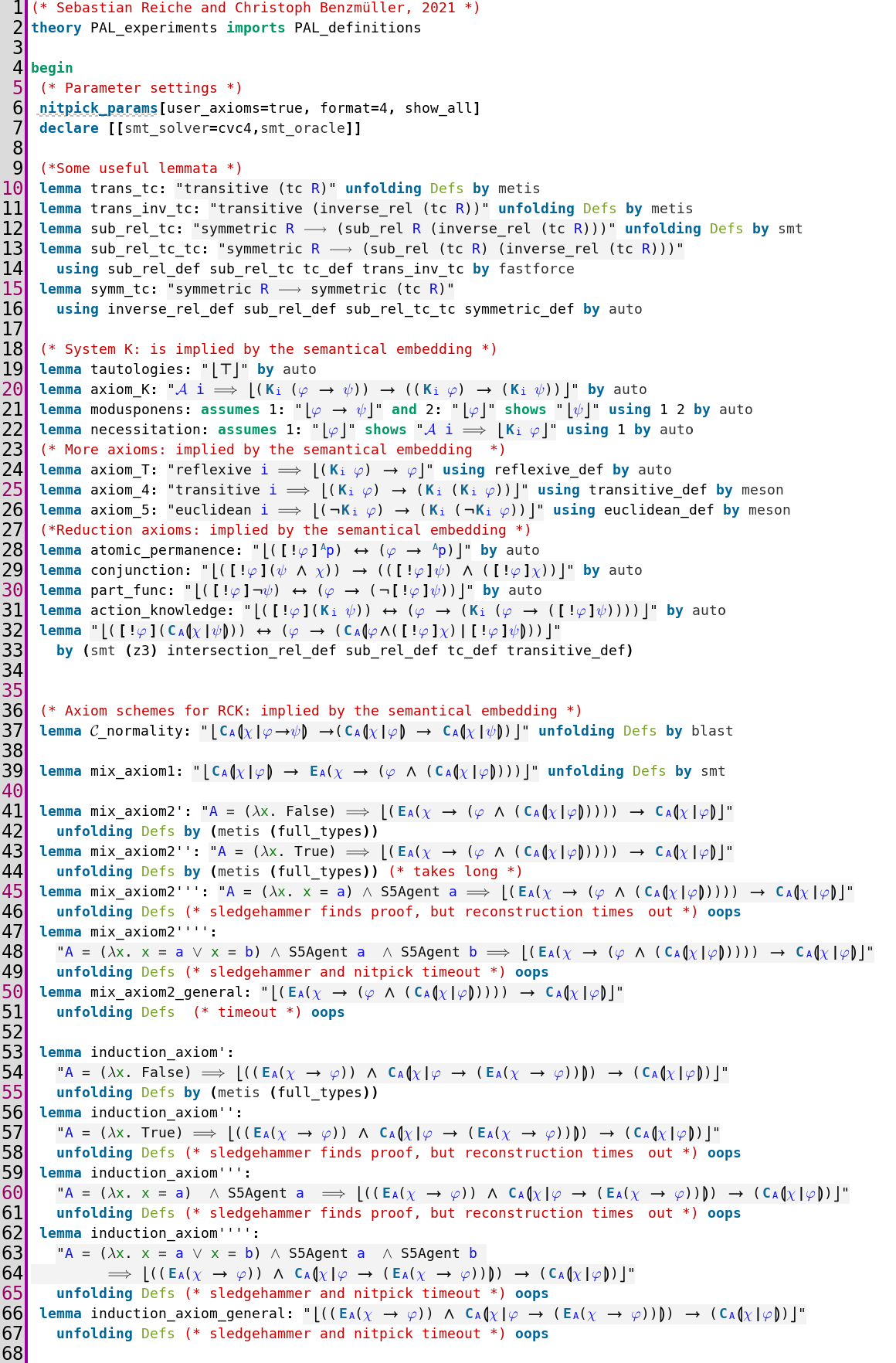}}
    \caption{Testing the automation of PAL in HOL}
    \label{fig:bild2}
\end{figure}
  
\begin{figure}[!htb]
    \centering
    \colorbox{gray!30}{\includegraphics[width=.95\textwidth]{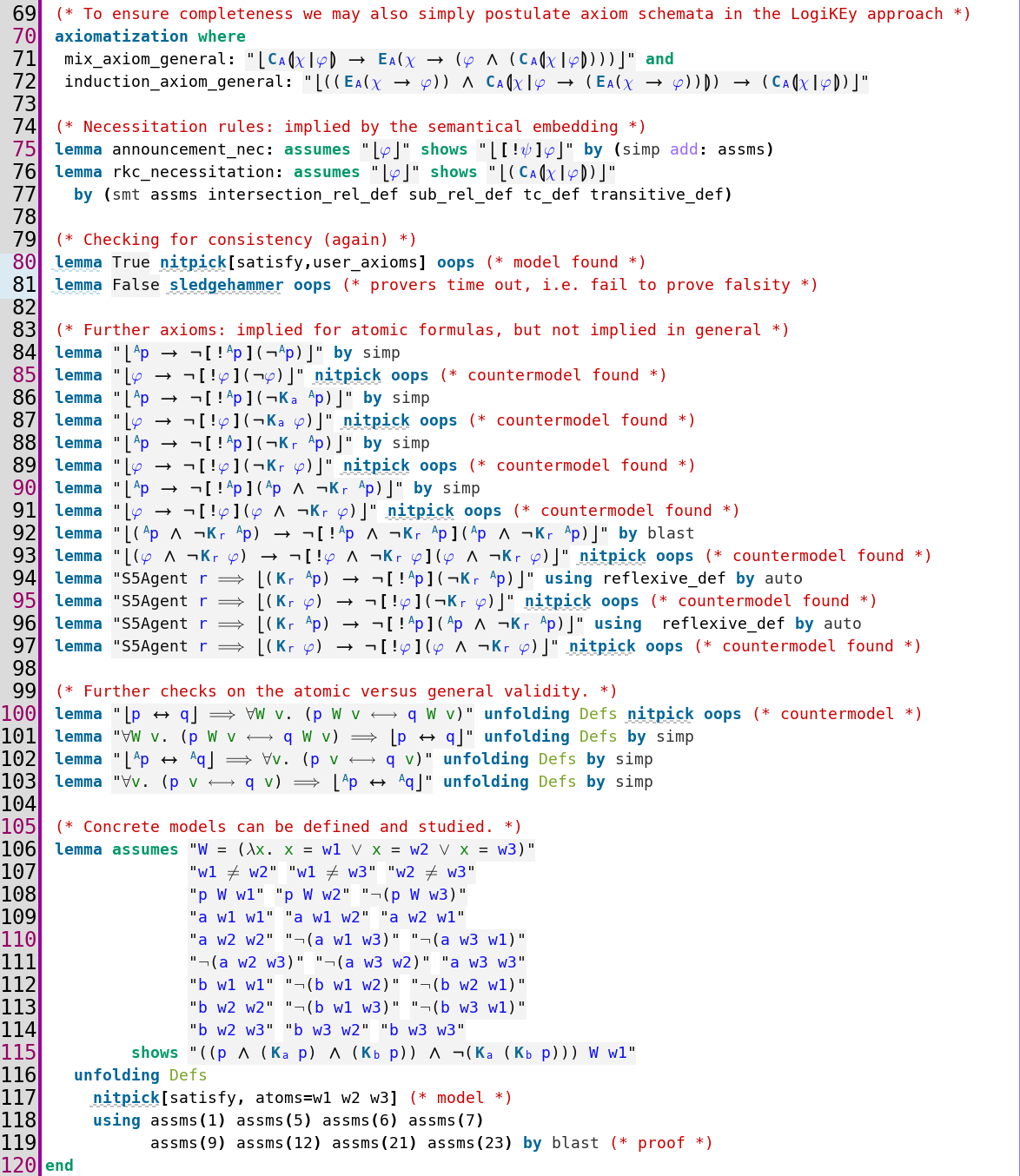}}
    \caption{Testing the automation of PAL in HOL (contd.)}
    \label{fig:bild3}
\end{figure}

\begin{figure}[!htb]
    \centering
    \colorbox{gray!30}{\includegraphics[width=.95\textwidth]{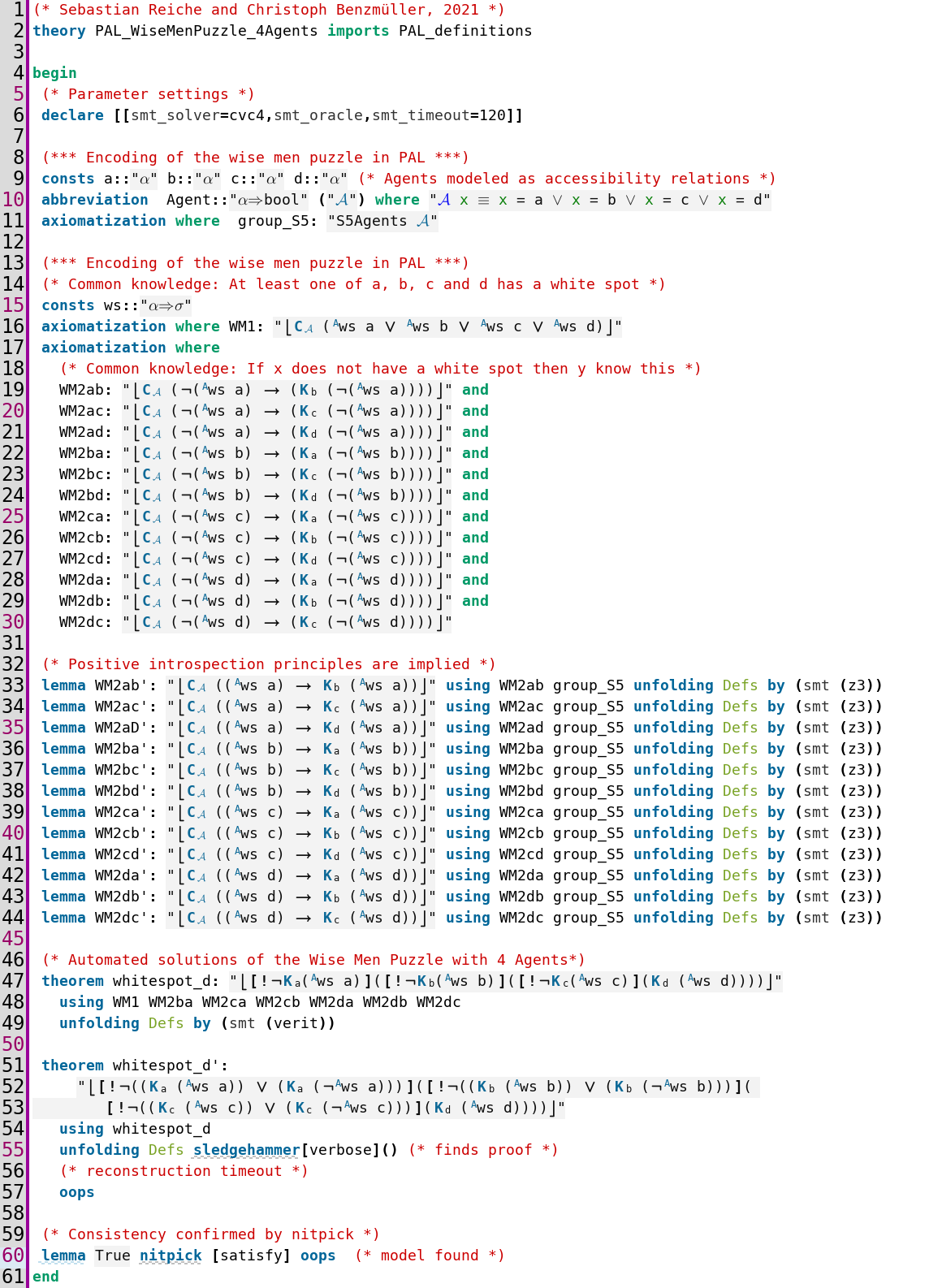}}
    \caption{Modeling and automating the wise men puzzle with four agents}
    \label{fig:bild4}
\end{figure}
  
\end{appendix}

\end{document}